\documentclass{siamart220329}  
\usepackage[T1]{fontenc}
\usepackage{times}

\usepackage{appendix}

\usepackage[height=8in]{geometry}
\usepackage{amsmath, amssymb, mathabx, comment, mathrsfs}
\usepackage{graphicx, color}
\usepackage{array, multirow}
\usepackage[font=small,labelfont=bf]{caption} 
\usepackage{lipsum}
\usepackage{amsfonts}
\usepackage{tabularx}
\usepackage{booktabs}
\usepackage{graphicx}
\usepackage{epstopdf}
\usepackage{mathtools}
\usepackage{dsfont}
\usepackage{subcaption}
\usepackage{tikz}
\usepackage{todonotes}
\definecolor{darkgreen}{rgb}{0, .5, 0}

\ifpdf
\DeclareGraphicsExtensions{.eps,.pdf,.png,.jpg}
\else
\DeclareGraphicsExtensions{.eps}
\fi

\AtBeginDocument{%
\setlength{\oddsidemargin}{\dimexpr(\paperwidth-\textwidth)/2-1in}%
\setlength{\evensidemargin}{\oddsidemargin}%
\setlength{\topmargin}{%
\dimexpr(\paperheight-\textheight)/2-\headheight-\headsep-1in}%
}

\usepackage[shortlabels]{enumitem}

\usepackage{algorithm}
\usepackage{algpseudocode}

\algnewcommand\algorithmicinput{\textbf{Input:}}
\algnewcommand\Input{\item[\algorithmicinput]}

\usepackage[sort]{natbib}

\setcitestyle{numbers,square,comma}
\usepackage{hyperref}
\hypersetup{colorlinks=true,
urlcolor=blue,
linkcolor=blue,
citecolor=blue,
bookmarksdepth=paragraph}
\usepackage[nameinlink]{cleveref}
\crefname{equation}{}{}
\crefname{section}{section}{sections}
\crefname{figure}{figure}{figures}
\crefname{table}{table}{tables}
\crefname{example}{example}{examples}
\crefname{proposition}{proposition}{propositions}
\Crefname{section}{Section}{Sections}
\Crefname{figure}{Figure}{Figures}
\Crefname{table}{Table}{Tables}
\Crefname{definition}{Definition}{Definitions}
\Crefname{theorem}{Theorem}{Theorems}
\Crefname{remark}{Remark}{Remarks}
\Crefname{example}{Example}{Examples}
\Crefname{proposition}{Proposition}{Propositions}
\numberwithin{equation}{section}

\usepackage{amsfonts,amsmath,amssymb}
\newsiamthm{remark}{Remark}
\newsiamthm{assumption}{Assumption}
\newsiamthm{example}{Example}
\newsiamremark{apriori}{A Priori Estimate}

\newcommand{\N}{\mathbb{N}}

\DeclareMathOperator*{\argmin}{arg\,min}

\NewDocumentCommand{\E}{o}{%
    \mathbb{E}^g\IfValueT{#1}{\big( {#1} \big)}
}
    \NewDocumentCommand{\Eg}{o}{%
        \mathbb{E}^g\IfValueT{#1}{\big( {#1} \big)}
    }

\renewcommand{\N}{\mathbb{N}}

\newcommand{\Rd}{\mathbb{R}^d}

\NewDocumentCommand{\Flow}{o}{%
    \operatorname{Flow}{\IfValueT{#1}{\left({#1}\right)}}
}
\NewDocumentCommand{\supp}{o}{%
    \operatorname{supp}{\IfValueT{#1}{\left({#1}\right)}}
}

\newcommand{\eqdef}{\ensuremath{\stackrel{\mbox{\upshape\tiny def.}}{=}}}

\usepackage{marginnote}
\definecolor{MidnightBlue}{RGB}{25,25,112}
\definecolor{MidnightBlueComplementingGreen}{RGB}{25,112,25}
\definecolor{MidnightBlueComplementingPurple}{RGB}{112,25,112}
\definecolor{MidnightBlueComplementingRed}{RGB}{112,25,69}
\definecolor{coolblack}{rgb}{0.0, 0.18, 0.39}

\definecolor{deepjunglegreen}{rgb}{0.0, 0.29, 0.29}
\definecolor{applegreen}{rgb}{0.55, 0.71, 0.0}
\definecolor{WowColor}{rgb}{.75,0,.75}
\definecolor{MildlyAlarming}{rgb}{0.85,0.25,0.1}
\definecolor{SubtleColor}{rgb}{0,0,.50}
\definecolor{SubtleColor2}{rgb}{0.6,0.21,.50}

\definecolor{lasallegreen}{rgb}{0.03, 0.47, 0.19}

\definecolor{amethyst}{rgb}{0.6, 0.4, 0.8}
\definecolor{auburn}{rgb}{0.43, 0.21, 0.1}
\definecolor{americanrose}{rgb}{1.0, 0.01, 0.24}
\definecolor{darkorchid}{rgb}{0.6, 0.2, 0.8}

\definecolor{persiangreen}{rgb}{0.0, 0.65, 0.58}
\newcounter{margincounter}

\NewDocumentCommand{\AK}{moo}{
    \IfValueF{#2}{
                {{\scriptsize
                            \textcolor{persiangreen}{
                                \textbf{AK:}
                                {#1}
                            }
                        }}
        }
    \IfValueT{#2}{\IfValueF{#3}{
                        \marginnote{{\scriptsize
                            \textcolor{persiangreen}{ 
                            \textbf{AK:}
                            \textit{{#1}}
                            }
                        }}
        }}
    \IfValueT{#3}{
                                    {{\scriptsize
                            \textcolor{persiangreen}{
                            \hfill\\
                                \noindent 
                                \textbf{AK:}
                                \textit{{#1}}
                            \hfill\\
                            }
                        }}
        }
                    }

\NewDocumentCommand{\MG}{moo}{
    \IfValueF{#2}{
                {{\scriptsize
                            \textcolor{darkorchid}{
                                \textbf{MG:}
                                {#1}
                            }
                        }}
        }
    \IfValueT{#2}{\IfValueF{#3}{
                        \marginnote{{\scriptsize
                            \textcolor{darkorchid}{ 
                            \textbf{MG:}
                            \textit{{#1}}
                            }
                        }}
        }}
    \IfValueT{#3}{
                                    {{\scriptsize
                            \textcolor{darkorchid}{
                            \hfill\\
                                \noindent 
                                \textbf{MG:}
                                \textit{{#1}}
                            \hfill\\
                            }
                        }}
        }
                    }

\definecolor{britishracinggreen}{rgb}{0.0, 0.26, 0.15}

\title{Neural Operators Which Certifiably Simultaneously Solve Infinitely Many Convex Programs using Logarithmic Rank and Linear Complexity}

\author{
  Anastasis Kratsios%
  \thanks{Department of Mathematics, McMaster University and the Vector Institute, Canada 
    (\email{kratsioa@mcmaster.ca}).}
  \and
  Ariel Neufeld%
  \thanks{Division of Mathematical Sciences, Nanyang Technological University, Singapore 
    (\email{ariel.neufeld@ntu.edu.sg}).}
  \and
  Philipp Schmocker%
  \thanks{Department of Mathematics, ETH Zurich, Switzerland
    (\email{philipp.schmocker@math.ethz.ch}).}
}

\headers{PEOs Can Efficiently Solve Convex Optimization Problems}{Kratsios, Neufeld, and Schmocker} 

\date{}

\usepackage{cleveref}
\newcounter{termcounter}
\renewcommand{\thetermcounter}{\Roman{termcounter}}
\crefname{term}{term}{terms}
\creflabelformat{term}{#2\textup{(#1)}#3}

\makeatletter
\def\term{\@ifnextchar[\term@optarg\term@noarg}
\def\term@optarg[#1]#2{%
  \textup{#1}%
  \def\@currentlabel{#1}%
  \def\cref@currentlabel{[][2147483647][]#1}%
  \cref@label[term]{#2}}
\def\term@noarg#1{%
  \refstepcounter{termcounter}%
  \textup{(\thetermcounter)}%
  \cref@label[term]{#1}}
\makeatother

\begin{document}

\maketitle

\begin{abstract}
    Neural operators (NOs) are a class of deep learning models designed to simultaneously solve infinitely many related problems by casting them into an infinite-dimensional space, whereon these NOs operate. Such operator learning problems arise naturally in mathematical finance, for instance, when learning families of optimal portfolios or hedging strategies. 
    {\color{black}Two}
    significant gap{\color{black}s}, however, remains between theory and practice: 
    {\color{black}1) Most universal approximation guarantees do not describe practically-implementable NOs or are pure existence statements and 2)}
    worst-case parameter bounds from universal approximation theorems suggest that constructs NOs may require an unrealistically large number of parameters to solve most operator learning problems, which stands in direct opposition to a slew of experimental evidence. 
    This paper closes that gap 
    {\color{black}by exhibiting a short and simple training algorithm construct a}
    specific class of NOs, proximal equilibrium operators (PEOs), using (realistic) finite-dimensional deep equilibrium layers to solve families of convex optimization problems over a separable Hilbert space $X$. Here, the inputs are smooth, convex loss functions on $X$, and the outputs are the associated (approximate) solutions to the optimization problems defined by each input loss. 

    {\color{black}Our main result is a certificate showing that:}
	when the input losses lie in suitable infinite-dimensional set, our {\color{black} PEOs trained with our algorithm do in fact} uniformly approximate the corresponding solutions to arbitrary precision, with rank and width growing only logarithmically and the depth growing linearly in the reciprocal of the approximation error. We then validate both our theoretical results and the trainability of PEOs on three applications: non-linear PDEs, stochastic optimal control, and portfolio optimization and hedging problems in mathematical finance.
\end{abstract}

\vspace{-0.05cm}

\noindent \textbf{Keywords:} 
{\color{black}Certifiable Machine Learning, Training Algorithms,} 
Convex Optimization, Operator Learning, Stochastic Optimal Control, Non-linear PDEs, Portfolio Optimization, Quadratic Hedging.


\section{Introduction} 
\label{s:Intro}

    We consider a family of convex optimization problems of the form
	\vspace{-0.1cm}
	\begin{equation}
		\label{eq:optim}
		g \quad \mapsto \quad \underset{x \in X}{\operatorname*{argmin}} \, \ell_{f,g}(x) \quad \text{with} \quad \ell_{f,g}(x) \eqdef f(x)+g(x),
		\vspace{-0.1cm}
	\end{equation}
	where $X$ is a separable Hilbert space, $f: X \rightarrow (-\infty,\infty]$ is proper, convex, and lower semicontinuous, and $g: X \rightarrow \mathbb{R}$ is convex and continuously differentiable with Lipschitz continuous gradient. We are interested in the case where $f$ is fixed and $g$ varies over a family $\mathfrak{G}$, parametrizing the optimization problem~\eqref{eq:optim}. The goal of this paper is to approximate a corresponding loss-to-solution operator simultaneously over $g \in \mathfrak{G}$.
	
	Problems of the form \eqref{eq:optim} arise naturally in mathematical finance~\cite{Schweizer_1999_QuadraticHedging,Pham_2000_QuadraticHedging,Lim_2004_QuadraticHedging,Buehler_2019_DeepHedging,Gnoatto_2024_DeepQuadraticHedging,Neufeld_2022_ChaoticHedging,Ruf_2022_HedgingLinearRegressionNN,Salvi_2025_RoughKernelHedging}, stochastic optimal control~\cite{Soner_2005_StochasticOptimalControl,Fleming_1975_DetStochOptimalControl,Touzi_2014_OptimalStochasticControl,Li_2024_NeuralNetworkSOC,Bertsekas_2005_DynamicProgramming,Han_2018_DLforSOC,Chen_2018_OptimalControl,Bachouch_2022_DNNforSOC,Becker_2019_DeepOptimalStopping}, and parametric families of non-linear PDEs~\cite{BeckCheridito_2021_DeepSplitting,marwah2023deep,li2020fourier,lu2021learning,lanthaler2022error,Khoo_2021_SolvingParametricPDE,Geist_2021_ParametricDiffEq,Kutyniok_2022_NNParPDE,Bhattacharya_2021_ModelRedParPDE,Berner_2020_ParFamKolmPDE,Han_2017_SolvingHighDimPDEs,Marcati_2023_ExpConvDeepON}. The variable $x \in X$ may represent a portfolio, trading strategy, stochastic control, or coefficient of a PDE, while $g$ describes an instance-dependent objective, such as expected utility, hedging error, or a more general stochastic control criterion. The fixed component $f$ can encode features that are shared across the entire family of problems, including admissibility constraints, regularization terms or smoothness penalties. An important special case is the convex indicator function
	\vspace{-0.1cm}
	\[
		f(x) = \iota_C(x) \eqdef
		\begin{cases}
			0, & x \in C, \\
			\infty, & x \notin C,
		\end{cases}
		\vspace{-0.1cm}
	\]
	where $C \subseteq X$ is a nonempty closed convex set, for which \eqref{eq:optim} reduces to $g \mapsto {\operatorname*{argmin}}_{x \in C} \, g(x)$. Hence, admissibility constraints such as no-short-selling restrictions or bounds on initial capital can be represented explicitly through the non-smooth component $f$, rather than being learned indirectly from data.
	
	When many related problems of the form \eqref{eq:optim} have to be solved, it is desirable to amortize their computational cost by learning the loss-to-solution operator over the family $g \in \mathfrak{G}$, rather than solving each problem separately. Neural operators (NOs) provide a natural framework for this task: they are deep learning models designed to simultaneously solve infinitely many related problems in infinite-dimensional spaces. Unfortunately, there is a large gap between NO theory and practice. Existing approximation guarantees show that NOs can approximate broad classes of infinite-dimensional operators~\cite{kovachki2021universal,galimberti2022designing,korolev2022two,cuchiero23global,teichmann26,neufeld2023universal,kratsios2024mixture}, but worst-case complexity bounds suggest that they may require an exorbitant number of parameters~\cite{lanthaler2023operator,lanthaler2024operator}, unless the target operators possess particularly strong regularity~\cite{Marcati_2023_ExpConvDeepON,adcock2024optimal}. This stands in contrast to extensive experimental evidence showing that NOs can successfully approximate a variety of computational problems with feasible model sizes; see, e.g., \cite{li2020fourier,yang2021seismic,pathak2022fourcastnet,molinaro2023neural,kovachki2023neural,benitez2024out,azizzadenesheli2024neural,herde2024poseidon}.
	
	This paper addresses this gap for the structured class of loss-to-solution operators arising from \eqref{eq:optim} by exploiting the additional structure of the underlying solution operator. We identify an iterative structure underlying these potentially non-smooth optimization problems that NOs can exploit through their \emph{depth}, and design a neural operator tailored to this structure. More precisely, we introduce a \emph{Proximal Equilibrium Operator} (PEO), based on finite-dimensional deep equilibrium layers, whose internal logic approximately encodes the proximal forward-backward splitting algorithm~\cite{combettes2005signal,Bredies_FBSplitting_2008,guan_2021_forward}. The model uses proximal operators as multivariate implicit non-linear activation functions and is related to deep equilibrium models~\cite{bai2019deep}, monotone DEQs~\cite{winston2020monotone}, and equilibrium architectures based on fixed-point iterations~\cite{gabor2024positive,furuya2024simultaneously,marwah2023deep,feischl2025neural}. Thus, the PEO does not approximate the loss-to-solution map as an arbitrary infinite-dimensional operator, but instead learns to mimic the optimization procedure underlying \eqref{eq:optim}.

	\subsection{Main Contributions}

    {\color{black}
    We provide a short training algorithm (Algorithm~\ref{alg:get_PEO}) which, given a common $f$ (cf.~\eqref{eq:optim}) returns a PEO.
    }
	Our main result (Theorem~\ref{thrm:Main_AproximateSelection}) shows that PEOs 
    {\color{black} constructed algorithmically in this manner are certified to}
    approximate the \textit{loss-to-solution} mapping of any admissible $g$ in~\eqref{eq:optim} for the splittable convex optimization problem over $X$.
    Critically, when the admissible $g$ are sufficiently well-behaved (see~\eqref{eq:nice_gs}), the approximation can be achieved by PEOs whose depth grows at most linearly and whose rank and width grow at most logarithmically in the reciprocal of the approximation error $\varepsilon > 0$. Moreover, if both $f$ and all admissible $g$ are Lipschitz with a shared worst-case Lipschitz constant, our second main result (Theorem~\ref{thrm:Main_NearOptimization}) shows that the optimal value can be recovered to roughly the same precision as the approximation accuracy of the loss-to-solution operator. Hence, PEOs \emph{of controlled complexity} can approximately solve infinitely many (non-linear) splittable convex optimization problems to high accuracy, thereby bypassing known limitations of neural operator architectures when approximating arbitrary continuous or smooth solution operators~\cite{lanthaler2024operator,galimberti2022designing}. Our proof is based on approximately ``unrolling'' the proximal forward–backward splitting iterations of~\cite{combettes2005signal}, whose quantitative convergence theory has been developed in~\cite{Bredies_FBSplitting_2008,guan_2021_forward} by representing each iteration through a layer of the neural operator. Each of these results is predicated on the existence of continuous approximate ($\eta$-)selectors for the coefficient ($g$) to solution operator for each splittable convex optimization problem in~\eqref{eq:optim}, with slack parameter $\eta>0$ (Proposition~\ref{prop:Existence}).

\subsection{Secondary Contributions}
We apply our main results to problems in non-linear partial differential equations (PDEs) (Section~\ref{s:Applications__ss:PDEs}), stochastic optimal control and portfolio optimization (Section~\ref{s:Applications__ss:Control}), as well as to quadratic hedging with liquidity constraints (Section~\ref{s:Applications__ss:Finance}). Each application explains and derives the relevant family of (non-linear) convex splittable optimization problems and is accompanied by a numerical illustration showing the reproducibility of our theoretical claims in each setting. An additional finite-dimensional application is included in our supplementary material (Section~\ref{s:Applications__ss:Convex}).

\subsection{Related Work}
\label{s:Introduction__ss:RelatedWorks}
Our NO analysis resonates with recent efforts in scientific machine learning to embed algorithmic priors into learned architectures. These include PDE solvers using deep operator networks \cite{marwah2023deep,li2020fourier,lu2021learning}, neural realizations of classical schemes such as multi-grid or fixed-point constructions, e.g.~Cauchy-Lipschitz, Lax-Milgram, Newton-Kantorovich theorems (see~\cite{feischl2025neural}), and the design of non-linear mappings with guaranteed geometric convergence via non-linear Perron-Frobenius theory \cite{gabor2024positive}. We highlight connections between our neural operator architecture and the literature on operator learning in infinite dimensions~\cite{kovachki2021universal,lanthaler2022error,lu2021learning}, which typically faces the curse of dimensionality~\cite{lanthaler2024operator}; algorithm unrolling~\cite{monga2021unrolling,mohammad2025deep}, which encodes algorithmic structure directly into neural network layers; and monotone-operator-theoretic perspectives on deep equilibrium models (DEQs)~\cite{BauschkeCombettes_2017CABook}. Our work aligns with the results showing favourable approximation rates for neural operators which can emulate, or ``unroll'', classical numerical algorithms (cf.~\cite{feischl2025neural,schwab2026operatorsurrog,monga2021unrolling,heiss2023multilevelCNNs,lin2025deeplayerlimit,lin2026deeplayerlimit,chouzenoux25,gregor2010learning}) or more broadly the algorithmic approximation theory of~\cite{kratsios2026algorithmic}.

\subsection{Organization of Paper}
\label{s:Introduction__ss:PaperOrg}
Section~\ref{s:Prelim} compiles the preliminary background and notation required in the formulation of our main result and it introduces our PEO model. Section~\ref{s:Main} contains the existence of a continuous (approximate) loss-to-solution operator and our main approximation guarantees thereof. Section~\ref{s:Applications} contains worked out applications of our results to PDEs, stochastic optimal control, and mathematical finance. Section~\ref{s:Conclusion} contains a conclusion, whereas all proofs are relegated to Section~\ref{s:Proofs}. Additional background on proximal operators is included in our paper's supplementary material (see Appendix~\ref{a:supplmat}).

\section{Preliminaries}
\label{s:Prelim}
We now cover the background and terminology required to formulate our results.

\paragraph{Notation}
Let $\mathbb N=\{0,1,2,\ldots\}$ and $\mathbb N_+=\mathbb N\setminus\{0\}$.
Given a vector field $V:\Rd\to\Rd$, we denote its support by $\supp[V]\eqdef \overline{\{x\in \Rd:\, V(x)\neq 0\}}$ where $\bar{A}$ denotes the closure of a subset $A\subseteq \mathbb{R}^d$ in the norm topology. For each $N\in \N_+$, we define the $N$-simplex $\Delta_N\eqdef \{w\in [0,1]^N:\, \sum_{n=1}^N\, w_n=1\}$.
Henceforth, $X$ is a separable infinite-dimensional real Hilbert space with a fixed orthonormal basis $(e_j)_{j=0}^{\infty}$. We identify gradients with their Hilbert-space representatives. We write $\Gamma_0(X)$ for the set of proper, lower semicontinuous convex functions $f:X\to(-\infty,\infty]$ and set $\operatorname{dom}(f)\eqdef\{x:f(x)<\infty\}$. For $R\in\mathbb N$, set $E_R\eqdef\operatorname{span}\{e_0,\ldots,e_{R-1}\}$ and $P_Rx\eqdef\sum_{j=0}^{R-1}\langle x,e_j\rangle_Xe_j$, with $E_0=\{0\}$ and $P_0=0$. For $R\ge1$, $a=(a_0,\ldots,a_{R-1})\in\mathbb R^R$, and $x\in X$, we retain the synthesis and analysis notation
\begin{equation}
\label{eq:lifting_operator}
 a^{\uparrow:R}\eqdef\sum_{j=0}^{R-1}a_je_j,
 \qquad x^{\downarrow:R}\eqdef(\langle x,e_j\rangle_X)_{j=0}^{R-1}.
\end{equation}
Thus $(a^{\uparrow:R})^{\downarrow:R}=a$ and $(x^{\downarrow:R})^{\uparrow:R}=P_Rx$. In particular, the latter expression equals $x$ exactly when $x\in E_R$. The synthesis map is an isometry. The orthogonal decomposition $x=P_Rx+(I-P_R)x$ gives $\|P_Rx\|_X\le\|x\|_X$ and $\|(I-P_R)x\|_X\le\|x\|_X$. For $v\in E_K$ and $R\ge K$, the contraction bound gives $\|(I-P_R)x\|_X=\|(I-P_R)(x-v)\|_X\le\|x-v\|_X$. Density of the union of the $E_K$ therefore gives $P_Rx\to x$. Finite orthonormality gives $\|P_Rx\|_X^2=\sum_{j<R}|\langle x,e_j\rangle_X|^2$; taking the limit proves Parseval's identity.
We fix a probability space $(\Omega,\mathcal F,\mathbb P)$ for the application sections.

\paragraph{Topology on the Loss Space}
We write $C^1(X;\mathbb R)$ for the continuously Fr\'echet-differentiable real-valued functions on $X$. Equip this space with the locally convex topology $\tau$ generated by
\[
 p_K(g)\eqdef\sup_{x\in K}\bigl(|g(x)|+\|\nabla g(x)\|_X\bigr),
\]
where $K$ ranges over the nonempty compact subsets of $X$. Equip $C(X;X)$ with uniform convergence on compact subsets. The map
\begin{equation}
\label{eq:gradient_map}
 C^1(X;\mathbb R)\ni g\mapsto\nabla g\in C(X;X)
\end{equation}
is continuous, since $\sup_{x\in K}\|\nabla g(x)-\nabla\widetilde g(x)\|_X\le p_K(g-\widetilde g)$.

\paragraph{Subdifferentials and Scaled Proximal Maps}
For $f\in\Gamma_0(X)$, define $\partial f:X\to2^X$ by
\[
 \partial f(x)\eqdef\{v\in X:f(y)\ge f(x)+\langle v,y-x\rangle_X\text{ for every }y\in X\}
 \quad\text{if }x\in\operatorname{dom}(f),
\]
and set $\partial f(x)=\varnothing$ otherwise. For $h>0$, define
\begin{equation}
\label{eq:prox_f__definition}
 \operatorname{prox}_{hf}(u)\eqdef\operatorname*{argmin}_{v\in X}
 \left\{hf(v)+\frac12\|v-u\|_X^2\right\}.
\end{equation}
Existence, uniqueness, firm non-expansiveness, and the characterization
\begin{equation}
\label{eq:proximal_equivalence}
 p=\operatorname{prox}_{hf}(u)\quad\Longleftrightarrow\quad(u-p)/h\in\partial f(p)
\end{equation}
are proved in Lemma~\ref{lem:scaled_proximal_calculus}. Thus $\operatorname{prox}_{hf}=(\operatorname{id}_X+h\partial f)^{-1}$ is a resolvent.

\subsection{Our Proximal Equilibrium Operator}
\label{s:Prelim__ss:Notation}
Our goal is to construct a \textit{neural operator}, that is, a neural network architecture between infinite-dimensional spaces, whose non-linearity is realized within \textit{deep equilibrium layers}, whose outputs are defined implicitly as fixed points of non-linear transformations. In our setting, these layers are induced by the proximal operator $\operatorname{prox}_f:X\to X$ associated with $f$ (see \eqref{eq:prox_f__definition}). In general, however, evaluating this proximal operator may require genuinely infinite-dimensional computations and therefore need not be implementable on real-world machines.


For $f\in\Gamma_0(X)$, $h>0$, and $R\in\mathbb N_+$, define
\begin{equation}
\label{eq:activation_function}
 \sigma_{f,h,R}:E_R\to E_R,\qquad
 \sigma_{f,h,R}(u)\eqdef P_R\operatorname{prox}_{hf}(u).
\end{equation}
In the following, we fix $h,R$ and abbreviate $\sigma_f=\sigma_{f,h,R}$. We adopt an oracle model in which one call returns the $R$ coordinates of $\sigma_f(u)$ exactly for any $u\in E_R$, excluding its internal evaluation cost from the oracle-call complexity.

\begin{definition}[Proximal Equilibrium Operator]
\label{def:PEO}
For fixed $R,M,L\in\mathbb N_+$, $f\in\Gamma_0(X)$, and $h>0$, a \emph{Proximal Equilibrium Operator} $\mathcal G:C(X;\mathbb R)\to E_R$ is defined by $x^{(0)}=0$ and
\begin{align}
 v^{(l)}(g)&\eqdef\bigl(g(x^{(l)}+x_m^{(l)})\bigr)_{m=1}^M,\nonumber\\
 x^{(l+1)}&\eqdef\gamma^{(l)}x^{(l)}+(1-\gamma^{(l)})
 \sigma_{f,h,R}\!\left(\left[A^{(l)}x^{(l)\downarrow:R}+B^{(l)}v^{(l)}(g)+b^{(l)}\right]^{\uparrow:R}\right),\nonumber\\
 \mathcal G(g)&\eqdef x^{(L)},\label{eq:corrected_architecture}
\end{align}
where $R$ is the \emph{rank}, $M$ is the \emph{width}, $L$ is the \emph{depth}, $A^{(l)}\in\mathbb R^{R\times R}$ are the \emph{$x$-weights}, $B^{(l)}\in\mathbb R^{R\times M}$ are the \emph{$g$-weights}, $b^{(l)}\in\mathbb R^R$ are the \emph{biases}, $x_m^{(l)}\in E_R$ are the \emph{shifts}, $\gamma^{(l)}\in[0,1]$ are the \emph{gating coefficients}, and $h > 0$ is the \emph{step size}, for $l = 0,\ldots,L-1$ and $m=1,\ldots,M$. All matrices, shifts, biases, and gates are fixed independently of the input $g$.
\end{definition}

Importantly, unlike standard deep equilibrium layers for neural operators, e.g.~\cite{marwah2023deep}, the map $\sigma_f$ acts on the finitely parameterized subspace $E_R\subseteq X$, whereas the proximal operator in~\eqref{eq:prox_f__definition} may involve genuinely infinite-dimensional computations. Beyond its access to the input function $g$ through finitely many point evaluations, our architecture differs from standard neural operators in several subtle but important ways. Most notably, instead of a univariate activation acting pointwise, we employ a \textit{problem-dependent} multivariate activation: for each problem~\eqref{eq:optim}, the potentially non-differentiable component $f$ determines the projected proximal map $\sigma_f$, which introduces non-linearity into the layer updates.

We additionally incorporate a gated residual connection, allowing the current iterate to be propagated alongside the non-linear update at each layer. While this is partly motivated by the empirically~\cite{borde2024scalable} and theoretically observed regularizing effects of residual connections on the loss landscape~\cite{riedi2023singular}, its role in our architecture is more fundamental. As shown in the proofs below, the gated residual structure arises naturally from the forward--backward proximal splitting iterations~\cite{combettes2005signal} that our architecture approximately unrolls.


\section{Main Guarantee}
\label{s:Main}
Before stating our main result; some preliminary remarks are due.
For $\lambda>0$, define
\begin{equation}
\label{eq:Uniformly_FrechetGradientClass}
 \mathcal X_\lambda\eqdef\{g\in C^1(X;\mathbb R):g\text{ is convex and }\nabla g\text{ is }\lambda\text{-Lipschitz}\}.
\end{equation}
For $r\ge0$, retain the gradient-tail class
\begin{equation}
\label{eq:nice_gs}
 \mathcal X_\lambda(r)\eqdef\left\{g\in\mathcal X_\lambda:
 \sup_{x\in X}\sum_{j=R}^{\infty}|\langle\nabla g(x),e_j\rangle_X|^2
 \le r2^{-R/2}\text{ for every }R\in\mathbb N\right\}.
\end{equation}
The index $R=0$ is included. In particular, it bounds the entire gradient, not only its positive-rank tails.
\begin{assumption}[Uniform Splitting and Compression]
\label{ass:uniform_splitting}
Fix $\lambda>0$, $r,D,C_f\ge0$, $h\in(0,1/\lambda]$, $f\in\Gamma_0(X)$, and a nonempty family $\mathfrak G\subseteq\mathcal X_\lambda(r)$. Put $B\eqdef2D+2+h\sqrt r$. Assume that:
\begin{enumerate}[label=(\roman*)]
 \item For every $g\in\mathfrak G$, there is an $x_g^\star\in\operatorname*{argmin}_{x\in X}\ell_{f,g}(x)$ with $\|x_g^\star\|_X\le D$.
 \item For every $R\in\mathbb N_+$ and $u\in E_R$ with $\|u\|_X\le B$,
 \begin{equation}
 \label{eq:proximal_tail_assumption}
  \|(I-P_R)\operatorname{prox}_{hf}(u)\|_X\le C_f2^{-R}.
 \end{equation}
\end{enumerate}
\end{assumption}
The constants in this assumption are fixed independently of both accuracy parameters and of the input loss.
For $g\in\mathfrak G$, define the exact forward-backward trajectory by
\begin{equation}
\label{eq:FB_Splitting}
 x_0^g=0,\qquad x_{l+1}^g\eqdef\operatorname{prox}_{hf}\bigl(x_l^g-h\nabla g(x_l^g)\bigr),\quad l\in\mathbb N.
\end{equation}
\begin{proposition}[Continuous Approximate Solution Selector]
\label{prop:Existence}
Under Assumption~\ref{ass:uniform_splitting}(i), let $\eta>0$ and set
$L_\eta\eqdef\max\{1,\lceil D^2/(2h\eta)\rceil\}$ and $S_\eta(g)\eqdef x_{L_\eta}^g$. 
Then $S_\eta:(\mathfrak G,\tau)\to X$ is continuous and
\[
 \sup_{g\in\mathfrak G}\left[\ell_{f,g}(S_\eta(g))-\inf_{x\in X}\ell_{f,g}(x)\right]\le\eta.
\]
\end{proposition}

\subsection{The Main Result: {\color{black}Training} Algorithm {\color{black} With Certifiable} Approximation Guarantees}

Our main abstract approximation theorem considers the PEO constructed by Algorithm~\ref{alg:get_PEO}. In particular, the result is not merely existential, but concerns a concrete neural operator with explicitly specified parameters. We highlight again that such guarantees are currently only available for kernel models such as random neural networks~\cite{gonon2023approximation,neufeld2023universal} or highly specialized deep learning models~\cite{kratsios2025beyond,ghoukasian2026certifiablyinterpretabletrainingrelumlps}. 

\begin{algorithm}[H]
\caption{\textbf{PEO} Training with {\color{black}Certifiable Approximation Guarantees}.}
\label{alg:get_PEO}
\begin{algorithmic}[1]
\Input $f\in\Gamma_0(X)$, depth $L\in\mathbb N_+$, rank $R\in\mathbb N_+$, step size $h>0$, finite-difference scale $\delta>0$.
\Require Access to $\sigma_{f,h,R}:u\mapsto P_R\operatorname{prox}_{hf}(u)$ for $u\in E_R$.

\State $M\gets R+1$, $A\gets I_R$, and $B\gets\frac{h}{\delta}\begin{bmatrix}\mathbf 1_R&-\mathbf{1}_R\end{bmatrix}$.
\State $b\gets0_{\mathbb R^R}$ and $\gamma\gets0$.
\Comment{Fix the shared weights, bias, and gate}

\State $s_1\gets0$ and $s_{j+2}\gets\delta e_j$ for $j=0,\ldots,R-1$.
\Comment{Fix the sampling shifts}

\State Define $\mathbf v_g(x)\eqdef\bigl(g(x+s_m)\bigr)_{m=1}^{M}$ for $g\in C(X;\mathbb R)$ and $x\in E_R$.
\State Define $\mathcal T_g(x)\eqdef\sigma_{f,h,R}\!\left(\left[Ax^{\downarrow:R}+B\mathbf v_g(x)\right]^{\uparrow:R}\right)$.
\Comment{Build one layer}

\State $\mathcal G_0: g \mapsto 0$.
\For{$l=0,\ldots,L-1$}
    \State $\mathcal G_{l+1}: g\mapsto\mathcal T_g\big(\mathcal G_l(g)\big)$.
    \Comment{Reuse the same layer}
\EndFor

\State \Return $\mathcal G\eqdef\mathcal G_L:
C(X;\mathbb R)\to E_R$.
\end{algorithmic}
\end{algorithm}

\begin{theorem}[Algorithmic Approximation Theorem]
\label{thrm:Main_AproximateSelection}
Under Assumption~\ref{ass:uniform_splitting}(i)--(ii), fix $\eta>0$ and $\varepsilon\in(0,1]$. Fix $L=L_\eta$,
 $R\eqdef\max\{1,
 \lceil4\log_2(1+\frac{3h\sqrt r L}{\varepsilon})\rceil,
 \lceil\log_2(1+\frac{3C_fL}{\varepsilon})\rceil\}$, and $ 
 M\eqdef R+1,\delta\eqdef\frac{2\varepsilon}{3h\lambda L\sqrt R}$. Then, running Algorithm~\ref{alg:get_PEO} returns a PEO $\mathcal G_{\eta,\varepsilon}$ of rank $R$, depth $L$, and sampling level $M$ satisfying
\begin{equation}
\label{eq:thrm:Main_AproximateSelection__Approx}
 \sup_{g\in\mathfrak G}\|S_\eta(g)-\mathcal G_{\eta,\varepsilon}(g)\|_X\le\varepsilon.
\end{equation}
Moreover, for fixed $h,\lambda,r,D,C_f$, we have $
 L=\mathcal O(1+\eta^{-1})$ and $R,M=\mathcal O\bigl(1+\log(1+L/\varepsilon)\bigr)$. 
\end{theorem}

Having shown that we may reconstruct the optimizer, we are able to guarantee the optimal value. 

\begin{theorem}[Near Optimal Value]
\label{thrm:Main_NearOptimization}
Under Assumption~\ref{ass:uniform_splitting}(i)--(ii), suppose that $f$ is finite and $\lambda_f$-Lipschitz on $\{x\in X:\|x\|_X\le2D+1\}$, where $\lambda_f\ge0$. For $\eta>0$ and $\varepsilon\in(0,1]$, let $\mathcal G_{\eta,\varepsilon}$ be the operator in Theorem~\ref{thrm:Main_AproximateSelection}. Then
\begin{equation}
\label{eq:thrm:Main_AproximateSelection__LossOptima}
 \sup_{g\in\mathfrak G}\left[\ell_{f,g}(\mathcal G_{\eta,\varepsilon}(g))-\inf_{x\in X}\ell_{f,g}(x)\right]
 \le\eta+(\lambda_f+\sqrt r)\varepsilon.
\end{equation}
\end{theorem}

\subsection{Discussion}
\label{s:Discussion}
Before applying PEOs in numerical examples, we briefly discuss some of the subtler points of our theorem.

\paragraph{Why Approximate the $\eta$-Solution Operator Instead of the True Solution Operator?\nopunct}
A subtle but important point is that the continuity of the $\eta$-approximate solution operator $S_\eta$ allows it to be approximated by continuous objects, such as our PEO models in Theorem~\ref{thrm:Main_AproximateSelection}, even when the true solution operator may not be approximable in this way. 
Under the hypotheses of Theorem~\ref{thrm:Main_NearOptimization},
\[
 \sup_{g\in\mathfrak G}
 \left[\ell_{f,g}(\mathcal G_{\eta,\varepsilon}(g))-\inf_X\ell_{f,g}\right]
 \le\eta+(\lambda_f+\sqrt r)\varepsilon.
\]
 Note that $\eta$ may be chosen arbitrarily small.

\paragraph{Feasibility-Preserving Objective Recovery}
The state approximation in Theorem~\ref{thrm:Main_AproximateSelection} does not automatically yield objective approximation when $f=\iota_C$, because the finite-rank projection can destroy feasibility and make the objective infinite. Proposition~\ref{prop:feasible_objective} shows that $P_R(C)\subseteq C$ preserves feasibility and gives objective error at most $\eta+\sqrt r\,\varepsilon$, extending the main construction to compatible constrained problems.

\begin{proposition}[Feasibility-Preserving Objective Recovery]
\label{prop:feasible_objective}
Under Assumption~\ref{ass:uniform_splitting}(i)--(ii), let $f=\iota_C$ for a nonempty closed convex set $C\subseteq X$. Fix $\eta>0$ and $\varepsilon\in(0,1]$, and let $R$ and $\mathcal G_{\eta,\varepsilon}$ be the rank and operator constructed in Theorem~\ref{thrm:Main_AproximateSelection}. If $P_R(C)\subseteq C$, then $\mathcal G_{\eta,\varepsilon}(g)\in C$ for every $g\in\mathfrak G$, and
\[
 \sup_{g\in\mathfrak G}\left[g(\mathcal G_{\eta,\varepsilon}(g))-\inf_{x\in C}g(x)\right]
 \le\eta+\sqrt r\,\varepsilon.
\]
\end{proposition}

\paragraph{Contractive Logarithmic-Depth Solver Realization}
Geometric convergence under contraction is classical, but here it must survive the accumulated finite-difference and projection errors uniformly over the loss family. Proposition~\ref{prop:contractive_realization} shows that the main construction can approximate the unique true minimizer with logarithmic depth under uniform contraction, identifying a stronger regime than the general theorem's $\mathcal O(1+\eta^{-1})$ optimization-depth bound.

\begin{proposition}[Contractive Logarithmic-Depth Solver Realization]
\label{prop:contractive_realization}

Under Assumption~\ref{ass:uniform_splitting}(i)--(ii), suppose that $T_g\eqdef\operatorname{prox}_{hf}\circ(\operatorname{id}_X-h\nabla g)$ satisfies for every $g\in\mathfrak G$ and each $x,y\in X$ that
\begin{equation}
\label{eq:uniform_contraction}
 \|T_gx-T_gy\|_X\le q\|x-y\|_X
\end{equation}
for a fixed $q\in(0,1)$. Then, each problem has a unique minimizer $x_g^\star$. Moreover, for every $\varepsilon\in(0,1]$, set
\begin{equation}
\label{eq:contractive_parameters}
 \begin{split}
 L&\eqdef\max\left\{1,\left\lceil\frac{\log(1+2D/\varepsilon)}{\log(1/q)}\right\rceil\right\},\\
 R&\eqdef\max\left\{1,
 \left\lceil4\log_2\left(1+\frac{6h\sqrt r}{(1-q)\varepsilon}\right)\right\rceil,
 \left\lceil\log_2\left(1+\frac{6C_f}{(1-q)\varepsilon}\right)\right\rceil\right\},\\
 M&\eqdef R+1,\qquad \delta\eqdef\frac{(1-q)\varepsilon}{3h\lambda\sqrt R}.
 \end{split}
\end{equation}
Then, there exists a Proximal Equilibrium Operator $\mathcal G_\varepsilon$ with these parameters such that
\[
 \sup_{g\in\mathfrak G}\|\mathcal G_\varepsilon(g)-x_g^\star\|_X\le\varepsilon.
\]
In addition, for fixed structural parameters (including $q$), its depth, rank, and per-layer sampling level are $\mathcal O(1+\log(1/\varepsilon))$ as $\varepsilon\downarrow0$. Its total loss-oracle count is $\mathcal O((1+\log(1/\varepsilon))^2)$. In addition, condition~\eqref{eq:uniform_contraction} holds with $q=(1+h\mu)^{-1}$ whenever $f-\frac\mu2\|\cdot\|_X^2$ is convex for some $\mu>0$.
\end{proposition}

\section{Numerical Experiments}
\label{s:Applications}

We illustrate in three different numerical examples\footnote{All numerical experiments have been implemented in \texttt{Python} using the \texttt{Tensorflow} package and were executed on a high-performing computing (HPC) cluster provided by the Digital Research Alliance of Canada. The code can be found under the following link: \url{https://github.com/psc25/ProximalEquilibriumOperator}.} how Proximal Equilibrium Operators can be implemented on a computer to learn convex splitting problems of the form~\eqref{eq:optim}.

\subsection{Learning the solution of a parametric family of non-linear PDEs}
\label{s:Applications__ss:PDEs}

In the first example, we learn the parameter-to-solution map of a family of non-linear PDEs. For a measurable subset $U\subseteq\mathbb R^d$, an initial condition $y_0\in X \eqdef L^2(U)$, and a bounded linear operator $\mathcal A:X\to X$, we consider non-linear PDE
\begin{equation}
    \label{EqDefPDE}
    \begin{aligned}
        \partial_t y(t)+\mathcal{A}^*\mathcal{A} y(t)+q\circ y(t)&=0,\\
        y(0)&=y_0,
    \end{aligned}
\end{equation}
and aim to learn the parameter-to-solution map $\mathcal{A}\mapsto y$. Our numerical approximation is based on proximal time stepping. More generally, for $f\in\Gamma_0(X)$, we consider an evolution, whenever it exists, of the form
\begin{equation}
    \label{EqGradFlow}
    \partial_t y(t)\in-\partial f(y(t)).
\end{equation}
The solution $y: [0,\infty) \rightarrow X$ of \eqref{EqGradFlow} is called the gradient flow of $f$. Given a partition $0=t_0<t_1<\cdots$, an implicit Euler discretization of \eqref{EqGradFlow} leads to
\[
    \frac{y(t_{k+1})-y(t_k)}{t_{k+1}-t_k} \in \partial f(y(t_{k+1})), \qquad k \in \mathbb{N}_0.
\]
By using the characterization~\eqref{eq:proximal_equivalence}, we observe that
\[
    y(t_{k+1}) = \big(\operatorname{id}_X + (t_{k+1}-t_k) \partial f(y(t_{k+1}))\big)^{-1}(y(t_k)) = \operatorname{prox}_{(t_{k+1}-t_k)f}(y(t_k)), \qquad k \in \mathbb{N}_0.
\]
Hence, the proximal minimization algorithm coincides with the implicit Euler method for numerically solving
a gradient flow differential equation.

\begin{lemma}[Bounded-Operator Energy Calculus]
\label{lem:bounded_pde_energy}
Let $U\subseteq\mathbb R^d$ be measurable, $X=L^2(U)$, and let $\mathcal A:X\to X$ be bounded and linear. For $K\ge0$ let $q:\mathbb R\to\mathbb R$ be nondecreasing and $K$-Lipschitz with $q(0)=0$, and set $Q(s)\eqdef\int_0^s q(t)\,dt$. Then,
\[
 f(x)\eqdef\int_U Q(x(u))\,du,\qquad g(x)\eqdef\frac12\|\mathcal A x\|_X^2
\]
are finite convex continuously Fr\'echet-differentiable functions, with $\nabla f(x)=q\circ x$ and $\nabla g(x)=\mathcal A^*\mathcal A x$. The Lipschitz constants of their gradients are at most $K$ and $\|\mathcal A\|_{\mathrm{op}}^2$, respectively.
\end{lemma}
\begin{proof}
The inequalities $|q(s)|\le K|s|$ and $0\le Q(s)\le Ks^2/2$ follow by integrating the Lipschitz bound from zero and using monotonicity. Thus, $q\circ x\in L^2(U)$ and $Q\circ x\in L^1(U)$. For $s,t \in \mathbb{R}$, the fundamental theorem of calculus gives
\[
 |Q(s+t)-Q(s)-q(s)t|
 =\left|\int_0^1(q(s+vt)-q(s))t\,dv\right|\le Kt^2/2.
\]
Substitution of $s=x(u)$ and $t=v(u)$, followed by integration, yields
\[
 |f(x+v)-f(x)-\langle q\circ x,v\rangle_X|\le K\|v\|_X^2/2.
\]
This is Fr\'echet differentiability. The pointwise Lipschitz inequality for $q$, squared and integrated, gives $\|q\circ x-q\circ y\|_X\le K\|x-y\|_X$. Monotonicity gives $Q(b)\ge Q(a)+q(a)(b-a)$ by integration, both for $b\ge a$ and for $b<a$. Applying this supporting-line inequality at $a=(1-t)b+tc$ to $b,c$ and adding with weights $1-t,t$ proves convexity of $Q$ for $0\le t\le1$; integration proves convexity of $f$. Expansion gives
\[
 g(x+v)-g(x)=\langle\mathcal A^*\mathcal Ax,v\rangle_X+\frac12\|\mathcal Av\|_X^2.
\]
The remainder is at most $\|\mathcal A\|_{\mathrm{op}}^2\|v\|_X^2/2$. Also $\|\mathcal A^*\mathcal A(x-y)\|_X\le\|\mathcal A\|_{\mathrm{op}}^2\|x-y\|_X$. Convexity of $g$ follows by expanding the identity for the squared norm of a convex combination of $\mathcal Ax$ and $\mathcal Ay$.
\end{proof}

We now consider the following parameter-dependent semilinear PDE, corresponding to the volatility-dependent pricing problem in a Bachelier model with default risk (see \cite{duffie1996,bender2017} and \cite[Section~3.1]{E_2019_MLP}).

\begin{example}
Fix $T,\sigma>0$ and $X=L^2(\mathbb R)$. We consider the semilinear reaction--diffusion equation
\begin{equation}
\label{EqDefPDE2}
 \partial_t y_\sigma(t,u)-\frac{\sigma^2}{2}\partial_{uu}y_\sigma(t,u)
 +\frac12\min\{y_\sigma(t,u),0\}=0,
 \qquad y_\sigma(0,u)=5u e^{-u^2}.
\end{equation}
Here $q(s)=\frac12\min(s,0)$ and $Q(s)=\frac14\min(s,0)^2$. For the differentiation operator on $X=L^2(\mathbb R)$, take $\operatorname{dom}(\mathcal A)=H^1(\mathbb R)$ and $\mathcal A x=\frac{\sigma}{2}\,x'$. The quadratic energy is
\[
 g_\sigma(x)=\begin{cases}\frac{\sigma^2}{4}\|x'\|_{L^2(\mathbb R)}^2,&x\in H^1(\mathbb R),\\+\infty,&x\notin H^1(\mathbb R).\end{cases}
\]
It is not a real-valued $C^1(L^2(\mathbb R);\mathbb R)$ loss. Thus the formal expression $-\frac{\sigma^2}{2} x''$ must not be treated as a globally Lipschitz gradient on $L^2(\mathbb R)$. To see the unboundedness directly, choose a nonzero $\phi\in C_c^\infty(\mathbb R)$ and set $\phi_n(u)=\sqrt n\,\phi(nu)$. The change of variable $v=\sigma$ gives $\|\phi_n\|_2=\|\phi\|_2$ and $\|\mathcal A\phi_n\|_2=\frac{\sigma}{2}\,n\|\phi'\|_2$. Here $\|\phi'\|_2>0$, since a compactly supported function with zero derivative is zero. The ratio therefore tends to infinity. A bounded extension agreeing with $\mathcal A$ on its domain would bound these ratios, which is impossible. In a finite-dimensional spatial discretization with matrix $A_R$, the discrete energy $a\mapsto\frac12\|A_Ra\|_2^2$ has gradient $A_R^\top A_Ra$ and gradient-Lipschitz constant $\|A_R\|_{\mathrm{op}}^2$; this constant need not remain bounded as the resolution increases. Indeed, expansion at $a+v$ gives the linear term $\langle A_R^\top A_Ra,v\rangle$ and the remainder $\frac12\|A_Rv\|_2^2$, while $\|A_R^\top A_R(a-b)\|_2\le\|A_R\|_{\mathrm{op}}^2\|a-b\|_2$.

For $Q(s)\eqdef\frac14\min(s,0)^2$, the reaction is $q(s)=Q'(s)=\frac12\min(s,0)$. For $h>0$,
\[
 \operatorname{prox}_{hQ}(s)=\begin{cases}2s/(2+h),&s<0,\\s,&s\ge0.\end{cases}
\]
To verify the formula, differentiate $hQ(v)+\frac12(v-s)^2$. Its derivative is $h\min(v,0)/2+v-s$, a strictly increasing continuous function. For $s<0$ its zero is $v=2s/(2+h)<0$, and for $s\ge0$ its zero is $v=s\ge0$. Each zero is therefore the unique minimizer.
For $f(x)=\int_{\mathbb R}Q(x(u))\,du$, this gives
\[
 \operatorname{prox}_{hf}(x)(u)=x(u)-\frac{h}{2+h}\min(x(u),0)
 \quad\text{for almost every }u.
\]
The proposed function is measurable and has absolute value at most $|x(u)|$, so it belongs to $L^2$. Its pointwise optimality inequality integrates to the proximal optimality inequality over $L^2$. The integral objective is strictly convex because of its quadratic term, so the integrated minimizer is unique. In particular, at $h=1$ the coefficient is $1/3$, not $1/4$.

For this example, both energy terms are nonnegative and vanish at zero.
For loss evaluations on the chosen Hermite span, the finite energy can be written on all of $X$ as
\[
 g_{\sigma,R}(x) \eqdef \frac{\sigma^2}{4}\|(P_Rx)'\|_{L^2(\mathbb R)}^2.
\]
Every Hermite basis function is smooth and has square-integrable derivative, so the derivative is a linear map on this finite span. The displayed energy is finite and has a globally Lipschitz linear gradient, with a bound that can depend on $R$; it still fails the global bounded-gradient requirement whenever this quadratic form is nonzero. Its values agree with $g_\sigma$ at every queried point in $E_R$.

In this setting, we aim to learn the operator
\begin{equation}
\label{EqDefParPDE}
 [\sqrt{0.02},\sqrt{0.8})\ni\sigma\mapsto\mathcal S(g_\sigma)\eqdef\mathcal S_T(\sigma)\eqdef y_\sigma(T,\cdot)\in L^2(\mathbb R)
\end{equation}
by a Proximal Equilibrium Operator $\mathcal{G}$ of rank $R = 8$, depth $L = 10$, and width $M = 20$. To this end, we choose the Hermite-Gaussian functions $(e_j)_{j \in \mathbb{N}}$ as basis of $L^2(\mathbb{R})$, which are defined by $e_j(u) \eqdef \frac{H_j(u)}{(2^j j!)^{1/2}} \frac{e^{-u^2/2}}{\pi^{1/4}}$, for $u \in \mathbb{R}$, where $(H_j)_{j \in \mathbb{N}}$ are the physicist's Hermite polynomials (see \cite[Eq.~22.2.14]{Abramowitz_1970_Handbook}). Moreover, we apply the Adam algorithm over 20000 epochs with learning rate $10^{-4}$ to train the Proximal Equilibrium Operator on a training set consisting of $400$ randomly initialized parameters $\sigma_1,...,\sigma_{400} \in [\sqrt{0.02},\sqrt{0.8})$. In addition, we evaluate its generalization performance every 250-th epoch on a test set consisting of $100$ randomly initialized parameters $\sigma_{401},...,\sigma_{500} \in [\sqrt{0.02},\sqrt{0.8})$. Hereby, the reference solution $y(T,\cdot) \eqdef \mathcal{S}(g_\sigma)$ of the non-linear PDE~\eqref{EqDefPDE2} is approximated by using a Multilevel Picard (MLP) algorithm (see, e.g., \cite{E_2019_MLP,E_2021_MLP}). The results are reported in Figure~\ref{FigNonLinearPDE}.
\end{example}

\begin{figure}[ht]
    \begin{minipage}[t]{0.49\textwidth}
        \centering
        \includegraphics[width=1.0\linewidth]{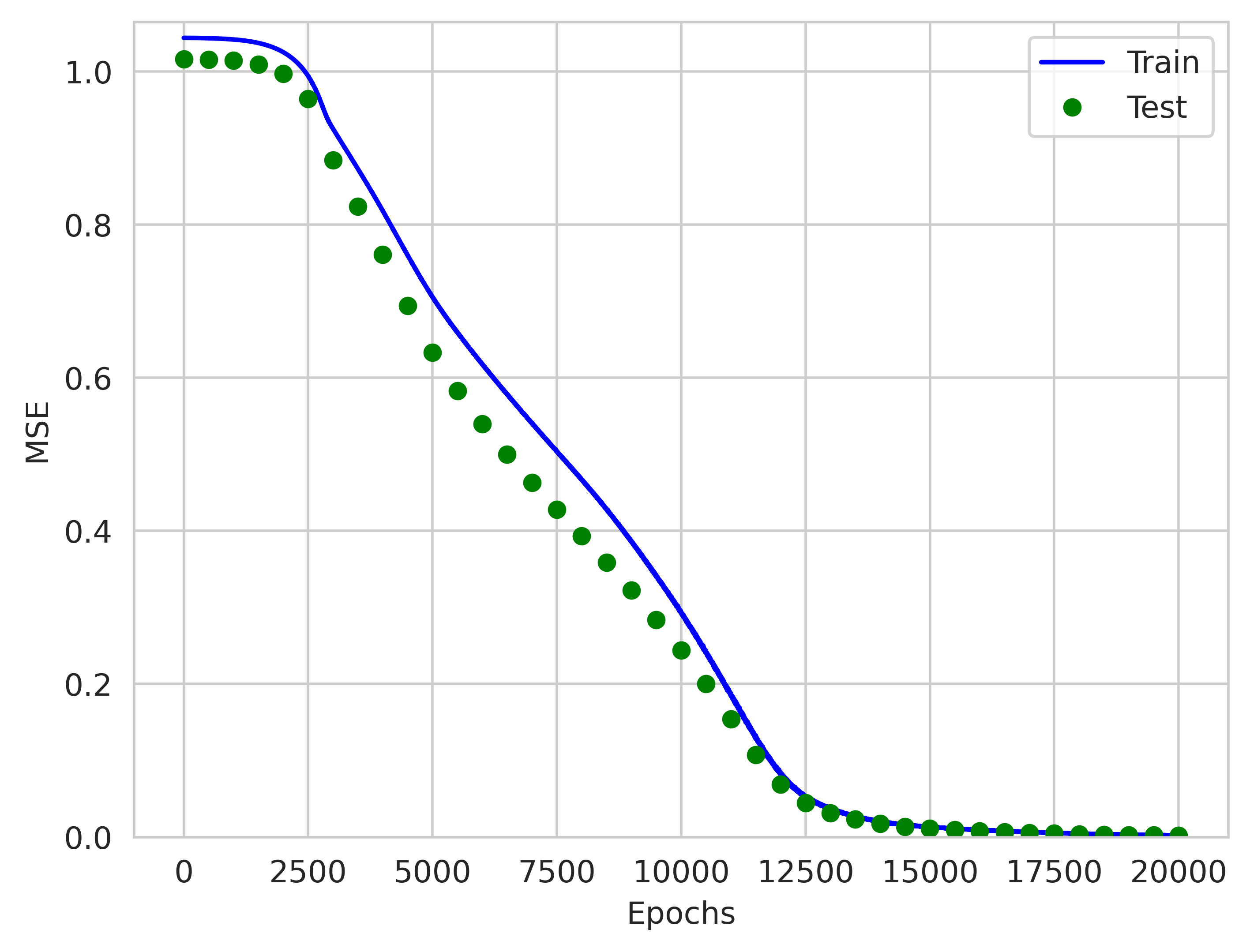}

        \subcaption{Learning performance}
    \end{minipage}
    \begin{minipage}[t]{0.49\textwidth}
        \centering
        \includegraphics[width=1.0\linewidth]{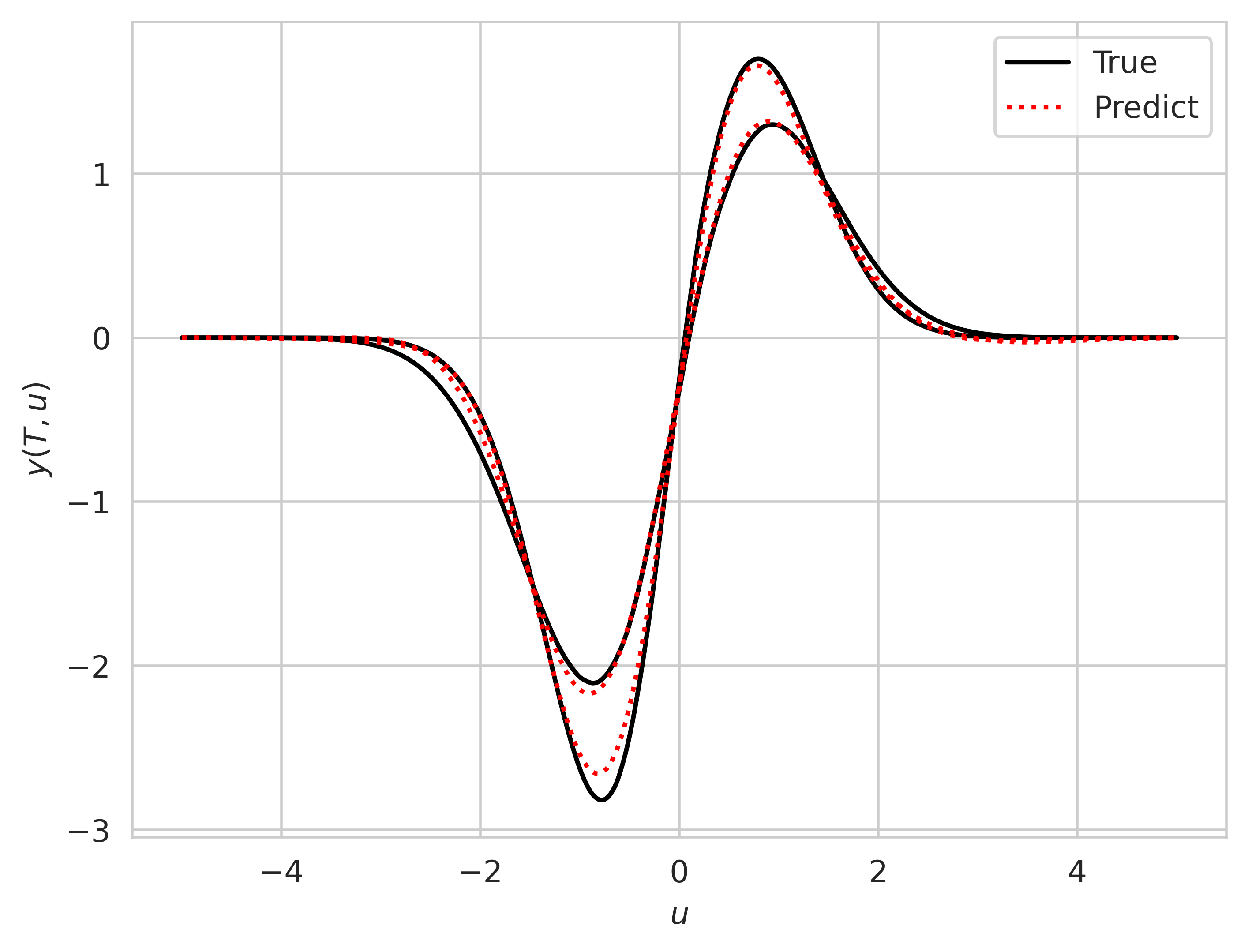}

        \subcaption{Solution of the PDE~\eqref{EqDefPDE2} for two $\sigma_k$ of the test set}
    \end{minipage}
    \caption{Learning the map \eqref{EqDefParPDE} returning the solution of the parametric PDE~\eqref{EqDefPDE2} by a Proximal Equilibrium Operator~$\mathcal{G}$. In (a), the learning performance is displayed in terms of the mean squared error (MSE) $\frac{1}{\vert K \vert} \sum_{k \in K} \Vert \mathcal{S}(g_{\sigma_k}) - \mathcal{G}(g_{\sigma_k}) \Vert^2$ on the training set (label ``Train'') and test set (label ``Test''). In (b), the predicted solution $\mathcal{G}(g_{\sigma_k})$ (label ``Predict'') is compared to the true solution $y(T,\cdot) = \mathcal{S}(g_{\sigma_k})$ (label ``True'') for two $\sigma_k$ of the test set.}
    \label{FigNonLinearPDE}
\end{figure}

\subsection{Stochastic optimal control}
\label{s:Applications__ss:Control}

Fix $T > 0$ and a filtered probability space $(\Omega,\mathcal{A},\mathbb{F},\mathbb{P})$ with filtration $\mathbb{F} \eqdef (\mathcal{F}_t)_{t \in [0,T]}$ satisfying the usual conditions. Let $X=L^2_{\mathrm{pred}}([0,T]\times\Omega;\mathbb R^n)$ denote the (equivalence classes of) predictable processes with
\[
 \|x\|_X=\left(\mathbb E\int_0^T\|x_t\|_{\mathbb R^n}^2\,dt\right)^{1/2} < \infty,
\]
which represent our controls. For each $x \in X$, assume a unique strong controlled state $y^x$ is defined by
\[
 dy_t^x=\mu(t,y_t^x,x_t)\,dt+\sigma(t,y_t^x,x_t)\,dW_t,
\]
with $y_0^x \eqdef y_0 \in \mathbb{R}^d$, where $W$ is a $d$-dimensional Brownian motion relative to the stated filtration and the jointly Borel coefficients have types $\mu:[0,T]\times\mathbb R^d\times\mathbb R^n\to\mathbb R^d$, and $\sigma:[0,T]\times\mathbb R^d\times\mathbb R^n\to\mathbb R^{d\times d}$. For Borel running and terminal utilities $c:[0,T]\times\mathbb R^d\times\mathbb R^n\to\mathbb R$ and $u:\mathbb R^d\to\mathbb R$, restrict to controls satisfying $\mathbb E\big[\int_0^T|c(t,y_t^x,x_t)|\,dt+|u(y_T^x)| \big]<\infty$ and define
\[
 g(x)\eqdef-\mathbb E\left[\int_0^T c(t,y_t^x,x_t)\,dt+u(y_T^x)\right].
\]
\begin{example}

We consider Merton's investment model on $[0,T]$, where $T>0$, $r,\sigma>0$, $\mu\in\mathbb R$, $S_0>0$, and $B_0=1$. Let $W$ be a one-dimensional Brownian motion relative to its completed natural filtration. The stock and bond dynamics are
\[
 dS_t=S_t(\mu\,dt+\sigma\,dW_t),\qquad B_t=e^{rt}.
\]
For a predictable dollar position $x\in X=L^2_{\mathrm{pred}}([0,T]\times\Omega;\mathbb R)$ and initial capital $y_0=1$, take the discounted self-financing wealth relation as the state model:
\[
 y_t^x=B_t\left[y_0+\int_0^t\frac{(\mu-r)x_s}{B_s}\,ds
                   +\int_0^t\frac{\sigma x_s}{B_s}\,dW_s\right].
\]
The last integral is the It\^o integral. Its integrand is predictable and satisfies
$\mathbb E\int_0^T|\sigma x_s/B_s|^2\,ds\le\sigma^2\|x\|_X^2<\infty$.
For the drift integral, Cauchy-Schwarz gives
\[
 \mathbb E\left(\int_0^T\left|\frac{(\mu-r)x_s}{B_s}\right|\,ds\right)^2
 \le(\mu-r)^2T\|x\|_X^2<\infty.
\]
These inequalities verify the integrability conditions in the displayed model. Restrictions needed for a particular utility are stated separately.

In the scalar predictable process space above, the nonnegativity constraint is
\[
 C\eqdef\{x\in X:x_t(\omega)\ge0\text{ for }dt\otimes d\mathbb P\text{-almost every }(t,\omega)\},
 \qquad f=\iota_C.
\]
Write $x_+(t,\omega)=\max\{x_t(\omega),0\}$ and $x_-(t,\omega)=\max\{-x_t(\omega),0\}$. The map $x\mapsto x_+$ preserves predictability and has norm at most $\|x\|_X$. For every $y\in C$, the scalar inequality $|x_+-x|\le|y-x|$ holds almost everywhere; squaring and integrating proves that $x_+$ minimizes $\|y-x\|_X^2$ over $C$. The minimizer is unique because this squared norm is strictly convex on the convex set $C$. Thus $\operatorname{prox}_{hf}(x)=x_+$ for every $h>0$.

The cone $C$ is convex because convex combinations of nonnegative representatives remain nonnegative almost everywhere. It is closed: the scalar negative-part map is $1$-Lipschitz, so $x_n\to x$ in $X$ with $x_n\in C$ gives $\|x_-\|_X\le\|x-x_n\|_X\to0$. Thus $f=\iota_C$ is proper, lower semicontinuous, and convex.
For a utility $u$ with its stated domain, write $g_u(x)=\mathbb E[-u(y_T^x)]$ only on controls for which this value is defined.
For the table calculation, set $\lambda\eqdef (\mu-r)/\sigma$, $B_t=e^{rt}$,
\[
 Z_t=\exp(-\lambda W_t-\lambda^2t/2),\qquad H_t=Z_t/B_t.
\]
The shifted utilities and their parameter conditions are
\begin{equation}
\label{eq:utility_reference_parameters}
 u(s)=\begin{cases}(s-x_0)^{1-\eta}/(1-\eta),&\eta\ne1,\\
 \log(s-x_0),&\eta=1,\end{cases}
 \quad s>x_0,\quad \eta>0,\quad x_0\ge0,\quad K=y_0-x_0/B_T>0.
\end{equation}
Thus $u'(s)=(s-x_0)^{-\eta}$. Its inverse $I:(0,\infty)\to(x_0,\infty)$ is $I(y)=y^{-1/\eta}+x_0$. For $y>0$,
\[
 \mathcal X(y)=\mathbb E[H_TI(yH_T)]
 =y^{-1/\eta}\mathbb E[H_T^{1-1/\eta}]+x_0/B_T.
\]
For every real $a$, the Gaussian exponential integral, evaluated by completing the square, gives
\[
 \mathbb E[H_T^a]=\exp\left(-arT+\frac{a^2-a}{2}\lambda^2T\right).
\]
Indeed, in the density integral for $W_T$,
$-a\lambda w-w^2/(2T)=-(w+a\lambda T)^2/(2T)+a^2\lambda^2T/2$;
the translated normal density integrates to one. In particular, $\mathbb E[H_T]=1/B_T$ and all real moments in the table are finite.

\(\mathcal X:(0,\infty)\to(x_0/B_T,\infty)\) is continuous and strictly decreasing: its positive moment coefficient multiplies the strictly decreasing function \(y^{-1/\eta}\); its limits at zero and infinity are infinity and \(x_0/B_T\), respectively. Hence it has an inverse \(\mathcal Y:(x_0/B_T,\infty)\to(0,\infty)\), and solving its displayed formula gives
\[
 \mathcal Y(w)=
 \left(\frac{\mathbb E[H_T^{1-1/\eta}]}{w-x_0/B_T}\right)^\eta,
 \qquad w>x_0/B_T.
\]
The moment is strictly positive because its integrand is positive, and its finiteness was proved above.

Put $a=1-1/\eta$ and $c=\lambda^2(1-\eta)/(2\eta^2)$. The candidate terminal wealth and its conditional budget are
\begin{equation}
\label{eq:utility_reference_candidate}
 \xi=x_0+\frac{K}{\mathbb E[H_T^a]}H_T^{-1/\eta},\qquad
 F_t=\mathbb E[H_T\xi\mid\mathcal F_t]
 =K Z_t^a e^{-ct}+x_0Z_t/B_T.
\end{equation}
For evaluation of the shifted utility at this candidate, use its positive surplus directly:
\[
 s_\xi\eqdef\frac{K}{\mathbb E[H_T^a]}H_T^{-1/\eta}>0,\qquad
 u(\xi)=
 \begin{cases}s_\xi^{1-\eta}/(1-\eta),&\eta\ne1,\\
 \log s_\xi,&\eta=1.
 \end{cases}
\]
This is the same surplus as $\xi-x_0$ in exact arithmetic, and supplies a direct evaluation formula.

To check the conditional formula, write $W_T=W_t+(W_T-W_t)$; the second summand is independent of $\mathcal F_t$ and is normal with variance $T-t$. Applying the preceding exponential integral to that increment gives
$\mathbb E[H_T^a\mid\mathcal F_t]/\mathbb E[H_T^a]=Z_t^a e^{-ct}$.
Also $\mathbb E[H_T\mid\mathcal F_t]=Z_t/B_T$ by the same calculation with exponent one. At $t=T$, these conditional identities follow by substituting the already measurable terminal variables.
Define the table's auxiliary expression $\psi_t \eqdef -\lambda aKZ_t^ae^{-ct}-\lambda x_0Z_t/B_T$. Then, $\psi_t+\lambda F_t=(\lambda/\eta)KZ_t^ae^{-ct}$ and the table's candidate portfolio is
\[
 x_t^*=\frac{\mu-r}{\sigma^2\eta}
 \frac{K Z_t^a e^{-ct}}{H_t}.
\]
Under $K>0$ its sign is the sign of $\mu-r$; thus nonnegativity of this candidate requires $\mu\ge r$.

\begin{lemma}[Conditional Verification of a Utility Reference]
\label{lem:utility_reference_verification}
For the scalar market model and the utility parameters in \eqref{eq:utility_reference_parameters}, let $\mathcal A_u$ be a class of controls with terminal wealth $Y^x>x_0$ almost surely, $\mathbb E|u(Y^x)|<\infty$, and $\mathbb E[H_TY^x]\le y_0$ for every $x\in\mathcal A_u$. If an admissible control $x^*\in\mathcal A_u$ has terminal wealth $Y^{x^*}=\xi$ from \eqref{eq:utility_reference_candidate}, then it maximizes $\mathbb E[u(Y^x)]$ over $\mathcal A_u$.
\end{lemma}
\begin{proof}
Set $b=(\mathbb E[H_T^a]/K)^\eta>0$. The definition of $\xi$ gives
$u'(\xi)=bH_T$ and $\mathbb E[H_T\xi]=x_0/B_T+K=y_0$.
All expectations of $u(\xi)$ are finite: for $\eta\ne1$ they are finite real powers of the positive lognormal variable $H_T$; for $\eta=1$, $\log(\xi-x_0)$ is an affine function of $W_T$ and has finite absolute expectation. Since $x_0\ge0$ and $Y^x>x_0$, the assumed budget bound also ensures integrability of $H_TY^x$. The derivative $u'(s)=(s-x_0)^{-\eta}$ is decreasing on the domain. Integrating it between $\xi$ and any $Y^x$ therefore gives the pointwise supporting-line inequality
\[
 u(Y^x)-u(\xi)\le u'(\xi)(Y^x-\xi)=bH_T(Y^x-\xi).
\]
All terms are integrable by the hypotheses and the preceding moment calculations. Taking expectations yields
\[
 \mathbb E[u(Y^x)]-\mathbb E[u(\xi)]
 \le b\bigl(\mathbb E[H_TY^x]-y_0\bigr)\le0.
\]
This proves the claimed conditional optimality.
\end{proof}
    \begin{table}[ht]
        \begin{subtable}[t]{\textwidth}
            \begin{tabular}{l|l|l|l|l|}
                $u(x)$ & $I(y)$ & $\mathcal{X}(y)$ & $\mathcal{Y}(y_0)$ & $\xi$ \\
                \hline
                $\begin{cases}
                    \frac{(x-x_0)^{1-\eta}}{1-\eta}, & \eta \neq 1, \\
                    \ln(x-x_0), & \eta = 1.
                \end{cases}$ & $\frac{1}{y^{1/\eta}} + x_0$ & $\frac{1}{y^{1/\eta}} \mathbb{E}\left[ H_T^{1-1/\eta} \right] + \frac{x_0}{B_T}$ & $\frac{\mathbb{E}\left[ H_T^{1-1/\eta} \right]^\eta}{\left( y_0 - x_0/B_T \right)^\eta}$ & $\frac{y_0 - x_0/B_T}{\mathbb{E}\left[ H_T^{1-1/\eta} \right] H_T^{1/\eta}} + x_0$ \\
                \hline
            \end{tabular}
        \end{subtable}

        \vspace{0.2cm}

        \begin{subtable}[t]{\textwidth}
            \begin{tabular}{l|l|l|l|l|}
                $\mathbb{E}[H_T \xi \vert \mathcal{F}_t]$ & $\psi_t$ & $x_t$ \\
                \hline
                $\left( y_0 - \frac{x_0}{B_T} \right) \frac{Z_t^{1-1/\eta}}{\exp\left( \frac{\lambda^2}{2} \frac{1-\eta}{\eta^2} t \right)} + \frac{x_0 Z_t}{B_T}$ & $-\lambda \left( y_0 - \frac{x_0}{B_T} \right) \frac{(1-1/\eta)Z_t^{1-1/\eta}}{\exp\left( \frac{\lambda^2}{2} \frac{1-\eta}{\eta^2} t \right)} - \lambda \frac{x_0 Z_t}{B_T}$ & $\frac{\mu-r}{\sigma^2 \eta} \frac{(y_0 - x_0/B_T) Z_t^{1-1/\eta}}{H_t \exp\left( \frac{\lambda^2}{2} \frac{1-\eta}{\eta^2} t \right)}$ \\
                \hline
            \end{tabular}
        \end{subtable}

        \caption{Computation of optimal portfolios for power and logarithmic utility functions given by $u(x) \eqdef \frac{(x-x_0)^{1-\eta}}{1-\eta}$ if $\eta \in (0,\infty) \setminus \lbrace 1 \rbrace$ and $u(x) \eqdef \ln(x-x_0)$ if $\eta = 1$, where $x_0 \geq 0$ is the reference point.}
        \label{TabOptimalPortfolio}
    \end{table}

All process equalities and nonnegativity constraints are interpreted $dt\otimes d\mathbb P$-almost everywhere. For a shifted utility $u:(x_0,\infty)\to\mathbb R$, let $\mathcal A_u$ consist of the specified admissible controls whose state is well defined, whose portfolio is nonnegative, whose terminal wealth exceeds $x_0$ almost surely, and for which $\mathbb E|u(y_T^x)|<\infty$. Define the solution set
\begin{equation}
\label{EqUtilityMax}
 \mathcal M_u=\operatorname*{argmax}_{x\in\mathcal A_u}\mathbb E[u(y_T^x)]\subseteq X.
\end{equation}
We study this optimization problem
by a Proximal Equilibrium Operator of rank $R = 10$, depth $L = 10$, and width $M = 40$. To this end, 
we use selected elements of the polynomial dictionary $\phi_{j_1,j_2}(t,\omega)=t^{j_1}W_t(\omega)^{j_2}$, $j_1,j_2\in\mathbb N$, with the polynomial convention $a^0=1$ also at $a=0$. Its elements are predictable because $(t,\omega)\mapsto(t,W_t(\omega))$ is continuous and adapted and the polynomials are continuous. For $t>0$, writing $W_t$ in law as $\sqrt t Z$ with $Z$ standard normal gives
\[
 \mathbb E\int_0^T|t^{j_1}W_t^{j_2}|^2\,dt
 =\mathbb E|Z|^{2j_2}\int_0^Tt^{2j_1+j_2}\,dt<\infty.
\]
The Gaussian moment is finite because its defining polynomial times Gaussian density is integrable; the time exponent is nonnegative. The single endpoint $t=0$ does not change the integral. This verifies membership for individual features without asserting completeness of the infinite dictionary.
From this infinite polynomial dictionary, select a finite dictionary for the architecture. To compare its coordinates with the theory, choose $R$ linearly independent elements $\phi_1,\ldots,\phi_R\in X$ from that dictionary, after verifying their membership in $X$. For $j=1,\ldots,R$, define
\[
 v_j\eqdef\phi_j-\sum_{k=1}^{j-1}\langle\phi_j,e_{k-1}\rangle_Xe_{k-1},
 \qquad e_{j-1}\eqdef v_j/\|v_j\|_X.
\]
The empty sum is zero. Induction gives orthogonality to the preceding vectors: for $i<j$, taking the inner product of $v_j$ with $e_{i-1}$ cancels exactly the $k=i$ summand. If $v_j=0$, then $\phi_j$ belongs to the span of its predecessors, contradicting independence. Division is therefore valid, and the resulting vectors are orthonormal with the same span as the dictionary.
 Moreover, we apply the Adam algorithm over 5000 epochs with learning rate $5 \cdot 10^{-5}$ and batchsize $100$ to train the Proximal Equilibrium Operator on a training set consisting of 400 utility functions $u_k(y) \eqdef \frac{(y-x_{0,k})^{1-\eta_k}}{1-\eta_k}$, $k = 1,...,400$. In addition, we evaluate its generalization performance every 125-th epoch on a test set consisting of 100 utility functions $u_k(y) \eqdef \frac{(y-x_{0,k})^{1-\eta_k}}{1-\eta_k}$, $k = 401,...,500$. Hereby, the risk aversion parameters $\eta_1,...,\eta_{500} \in [0.25,0.75)$ and the reference points $x_{0,1},...,x_{0,500} \in [0,\infty)$ are randomly initialized. The results are reported in Figure~\ref{FigUtilityMax}.
\end{example}

\begin{figure}[ht]
    \begin{minipage}[c]{0.49\textwidth}
        \centering
        \includegraphics[width=0.99\linewidth]{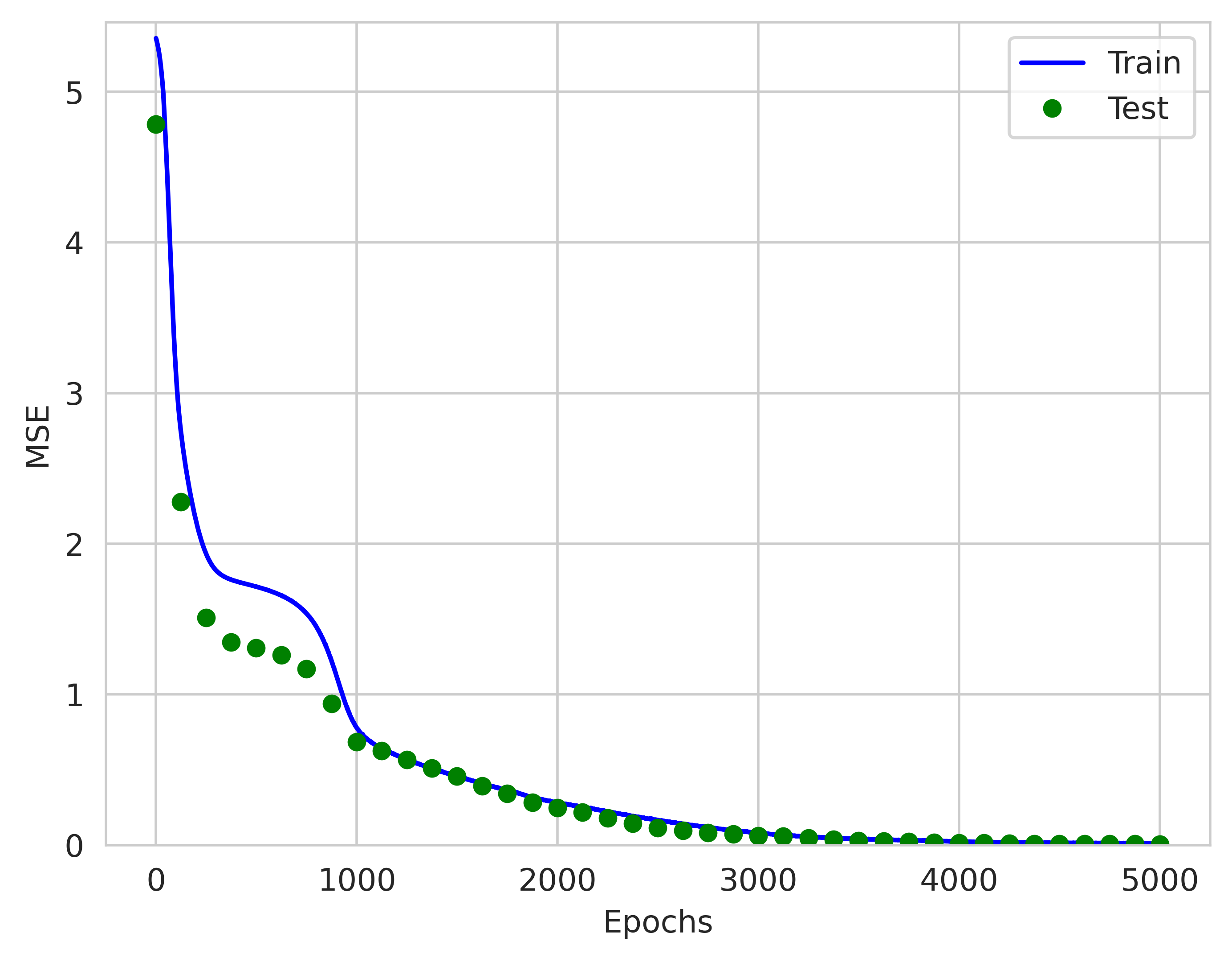}

        \subcaption{Learning performance}
    \end{minipage}
    \begin{minipage}[c]{0.49\textwidth}
        \centering
        \includegraphics[width=1.0\linewidth]{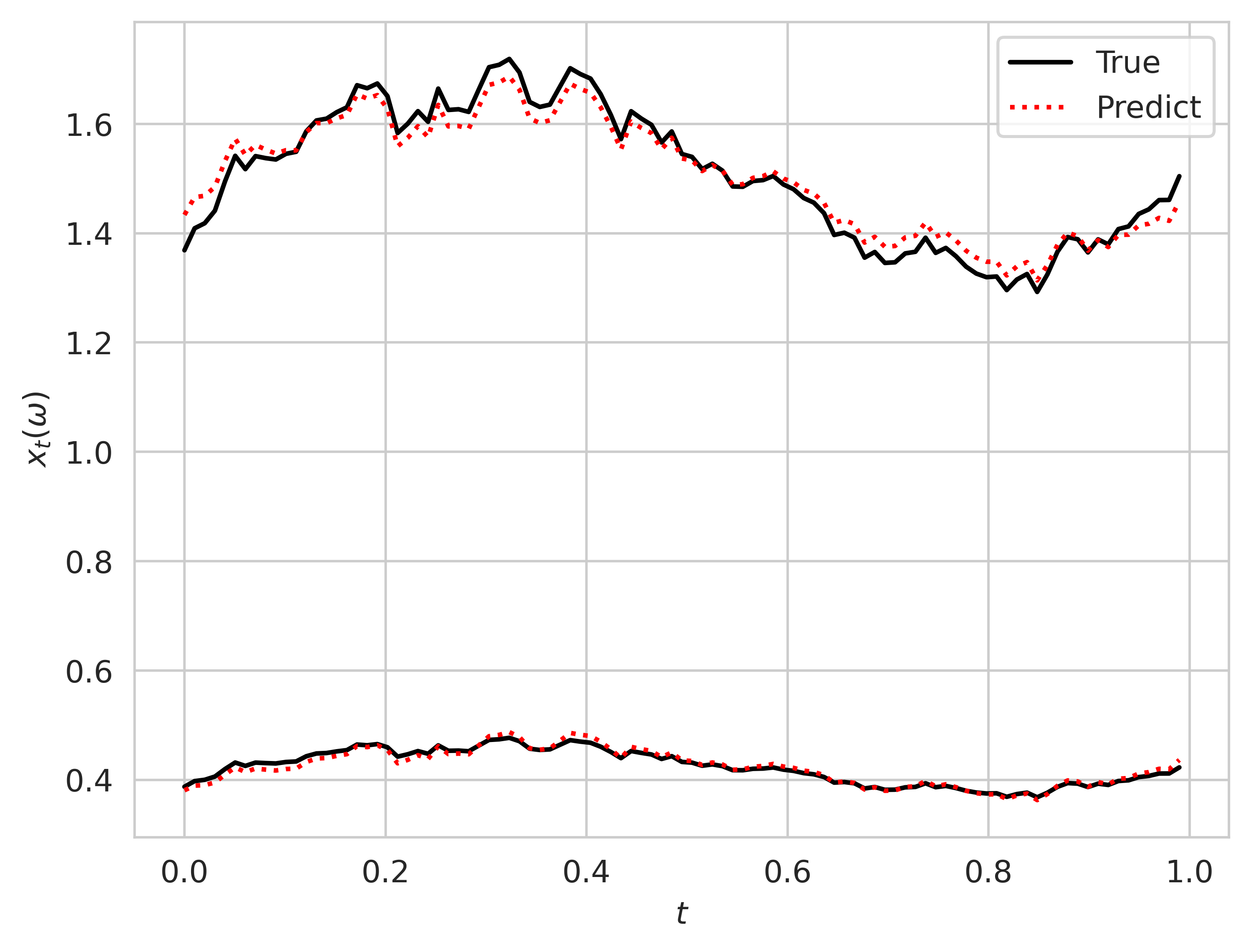}
        
        \subcaption{Optimal portfolio $x$ for two $\eta_k$ of the test set and some fixed $\omega \in \Omega$}
    \end{minipage}
    \caption{Learning the solution operator $\mathcal{S}$ of the utility maximization problem~\eqref{EqUtilityMax} by a Proximal Equilibrium Operator~$\mathcal{G}$. In (a), the learning performance is displayed in terms of the mean squared error (MSE) $\frac{1}{\vert K \vert} \sum_{k \in K} \Vert \mathcal{S}(g_{u_k}) - \mathcal{G}(g_{u_k}) \Vert^2$ on the training set (label ``Train'') and test set (label ``Test''). In (b), the predicted solution $\mathcal{G}(g_{u_k})$ (label ``Predict'') is compared to the true solution $\mathcal{S}(g_{u_k})$ (label ``True'') for two $\eta_k$ of the test set and some fixed $\omega \in \Omega$.}
    \label{FigUtilityMax}
\end{figure}

\subsection{Quadratic hedging with liquidity constraint}
\label{s:Applications__ss:Finance}

In an incomplete financial market model, we learn the pricing/hedging operator that returns for a given financial derivative an approximation of the optimal hedging strategy in the sense of quadratic hedging (see \cite{Schweizer_1999_QuadraticHedging,Pham_2000_QuadraticHedging}) under an additional liquidity constraint. For $T > 0$, a filtered probability space $(\Omega,\mathcal{A},\mathbb{F},\mathbb{P})$ with filtration $\mathbb{F} \eqdef (\mathcal{F}_t)_{t \in [0,T]}$ satisfying the usual conditions, and a continuous strictly positive semimartingale $S \eqdef (S_t)_{t \in [0,T]}$ with decomposition $S_t = S_0 + A_t + M_t$, $t \in [0,T]$, into a process of finite variation $A \eqdef (A_t)_{t \in [0,T]}$ and a local martingale $M \eqdef (M_t)_{t \in [0,T]}$ with $A_0 = M_0 = 0$, we assume the structure condition, i.e., the existence of a predictable process $\lambda \eqdef (\lambda_t)_{t \in [0,T]}$ with
\begin{equation*}
    A_t = \int_0^t \lambda_s \, d\langle M \rangle_s, \qquad t \in [0,T],
\end{equation*}
such that the mean--variance tradeoff process
\begin{equation*}
    K_t \eqdef \int_0^t \lambda_s^2 \, d\langle M \rangle_s, \qquad t \in [0,T],
\end{equation*}
satisfies $K_T \in L^\infty(\mathbb{P})$. Then, we consider the Hilbert space $\mathbb{R} \oplus L^2(S)$ of tuples $(x,\theta) \in \mathbb{R} \oplus L^2(S)$ equipped with the inner product $\langle (x,\theta),(y,\vartheta) \rangle_{\mathbb{R} \oplus L^2(S)} \eqdef xy + \langle \theta,\vartheta\rangle_{L^2(S)}$, where $L^2(S)$ denotes the Hilbert space of $\mathbb{F}$-predictable processes $\theta \eqdef (\theta_t)_{t \in [0,T]}$ satisfying $\mathbb{E}\big[ \int_0^T \theta_t^2 \, d\langle M\rangle_t \big] < \infty$, equipped with the inner product $\langle \theta,\vartheta\rangle_{L^2(S)} \eqdef \mathbb{E}\big[ \int_0^T \theta_t\vartheta_t \, d\langle M\rangle_t \big]$. Moreover, we introduce the terminal-gain operator
\begin{equation*}
    L^2(S) \ni \theta \mapsto \mathcal{I}(\theta) \eqdef \int_0^T \theta_t \, dS_t \in L^2(\mathbb{P}),
\end{equation*}
which is linear and, by the structure condition, the Cauchy--Schwarz inequality, and It\^o's isometry, bounded as
\begin{equation}
    \label{eq:I_bounded}
    \begin{aligned}
        \Vert \mathcal{I}(\theta) \Vert_{L^2(\mathbb{P})} & = \mathbb{E}\left[ \left( \int_0^T \theta_t \, dS_t \right)^2 \right]^\frac{1}{2} \leq \mathbb{E}\left[ \left( \int_0^T \theta_t \, dA_t \right)^2 \right]^\frac{1}{2} + \mathbb{E}\left[ \left( \int_0^T \theta_t \, dM_t \right)^2 \right]^\frac{1}{2} \\
        & \leq \mathbb{E}\left[ \left( \int_0^T \theta_t \lambda_t \, d\langle M \rangle_t \right)^2 \right]^\frac{1}{2} + \mathbb{E}\left[ \int_0^T \theta_t^2 \, d\langle M \rangle_t \right]^\frac{1}{2} \leq \left( \Vert K_T \Vert_{L^\infty(\mathbb{P})}^\frac{1}{2} + 1 \right)
        \Vert \theta \Vert_{L^2(S)}.
    \end{aligned}
\end{equation}
For a given liquidity constraint $C > 0$, we aim to learn the quadratic hedging operator
\begin{equation*}
    L^2(\mathbb{P}) \ni H \mapsto \!\argmin_{(x,\theta) \in \mathbb{R} \oplus L^2(S) \atop x \in [0,C]} \!\mathbb{E}\!\left[ \left( H \!-\! x \!-\! \int_0^T \theta_t \, dS_t \right)^2 \right] = \!\argmin_{(x,\theta) \in \mathbb{R} \oplus L^2(S)}\! \left( f(x,\theta) \!+\! g_H(x,\theta) \right) \in \mathbb{R} \oplus L^2(S),
\end{equation*}
where
\begin{equation*}
    f(x,\theta) \eqdef
    \begin{cases}
        0 & \text{if } x \in [0,C], \\
        \infty & \text{otherwise},
    \end{cases}
\end{equation*}
and
\begin{equation*}
    g_H(x,\theta) \eqdef \mathbb{E}\left[ \left( H-x-\int_0^T \theta_t \, dS_t \right)^2 \right] = \Vert H-x - \mathcal{I}(\theta) \Vert_{L^2(\mathbb{P})}^2.
\end{equation*}
Hereby, we observe that the proximal operator of $f:\mathbb{R}\oplus L^2(S)\rightarrow(-\infty,\infty]$ is given by $\operatorname{prox}_f(x,\theta) = \big( \operatorname{proj}_{[0,C]}(x), (\theta_t)_{t\in[0,T]} \big)$, with $\operatorname{proj}_{[0,C]}(s) \eqdef \argmin_{t\in[0,C]}|s-t| = \max(\min(s,C),0)$, while $g_H:\mathbb{R}\oplus L^2(S)\rightarrow\mathbb{R}$ is convex and continuously Fr\'echet differentiable, with
\begin{equation*}
    \begin{aligned}
        \nabla g_H(x,\theta) = 2 \big( x+\mathbb{E}\left[\mathcal{I}(\theta)-H\right], \mathcal{I}^* \left(x+\mathcal{I}(\theta)-H \right) \big),
    \end{aligned}
\end{equation*}
where $\mathcal{I}^*: L^2(\mathbb{P}) \rightarrow L^2(S)$ denotes the adjoint of $\mathcal{I}$. In particular, \eqref{eq:I_bounded} implies
\begin{equation*}
    \operatorname{Lip}(\nabla g_H) \leq 2\left( 1+\Vert\mathcal{I}\Vert_{\mathrm{op}}^2 \right) \leq 2 \left( 1+ \left( 1+\Vert K_T\Vert_{L^\infty(\mathbb{P})}^\frac{1}{2} \right)^2 \right).
\end{equation*}

\begin{example}
    We consider the Heston model with stock price $S \eqdef (S_t)_{t \in [0,T]}$ and stochastic volatility $V \eqdef (V_t)_{t \in [0,T]}$ following the SDEs
    \begin{equation*}
        \begin{aligned}
            dS_t & = \sqrt{V_t} S_t dW^1_t, \\
            dV_t & = \kappa (\theta - V_t) dt + \sigma \sqrt{V_t} dW^2_t,
        \end{aligned}
    \end{equation*}
    with initial values $S_0,V_0>0$ and parameters $\kappa,\theta,\sigma > 0$ as well as $d\langle W^1, W^2 \rangle_t = \rho dt$ for some $\rho \in [-1,1]$. Let $\mathbb F$ be the completed filtration generated by $W^1,W^2$. For simplicity, we restrict ourselves to the vector subspace $\mathcal{H} \subseteq L^2(S)$ of $(S,V)$-Markovian strategies $\theta \in L^2(S)$, which have a representation $\theta_t=\vartheta(t,S_t,V_t)$ for a Borel function $\vartheta$. For $H\in L^2(\mathbb P)$, the constrained solution set is
    \begin{equation}
    \label{EqQuadHedg}
     \mathcal M_H=\operatorname*{argmin}_{(x,\theta)\in[0,C]\times L^2(S)}
     \|H-x-\mathcal I\theta\|_{L^2(\mathbb P)}^2\subseteq\mathbb R\oplus\mathcal H.
    \end{equation}
We study this optimization problem
by a Proximal Equilibrium Operator $\mathcal{G}$ of rank $R = 11$, depth $L = 10$, and width $M = 40$. 
The strategy dictionary consists of selected functions
\[
 \psi_{j_1,j_2,j_3}(t)=t^{j_1}\bigl(\log(S_t/S_0)\bigr)^{j_2}(V_t/V_0)^{j_3},
\qquad j_1,j_2,j_3\in\mathbb N.
\]
Use the polynomial convention $a^0=1$, including at $a=0$. Their membership in $\mathcal H$ requires the corresponding moments in its specified norm.
In the product space $X=\mathbb R\oplus\mathcal H$, reserve one coordinate for capital. For $R\ge1$, choose $R-1$ linearly independent predictable dictionary elements $\psi_1,\ldots,\psi_{R-1}\in\mathcal H$, after verifying the required moments. If $R=1$, this list is empty. For $j=1,\ldots,R-1$, set
\[
 w_j\eqdef\psi_j-\sum_{k=1}^{j-1}\langle\psi_j,u_k\rangle_{\mathcal H}u_k,
 \qquad u_j\eqdef w_j/\|w_j\|_{\mathcal H}.
\]
Taking the inner product with each preceding $u_k$ shows that $w_j$ is orthogonal to them. Independence makes $w_j\ne0$, so these formulas give an orthonormal family spanning the chosen strategy dictionary. Set $e_0=(1,0)$ and $e_j=(0,u_j)$ for $1\le j<R$. Then the projection onto their span is
\[
 P_R(x,\theta)=\left(x,\sum_{j=1}^{R-1}\langle\theta,u_j\rangle_{\mathcal H}u_j\right).
\]
Thus $P_R$ preserves the capital constraint $x\in[0,C]$. For $f=\iota_{[0,C]\times\mathcal H}$, minimization of the two orthogonal squared-distance terms gives
\[
 \operatorname{prox}_{hf}(x,\theta)=\bigl(\operatorname{proj}_{[0,C]}x,\theta\bigr),\qquad h>0.
\]
Consequently, the proximal map preserves this finite-dimensional subspace. For comparisons across ranks, use prefixes of one fixed verified dictionary and the resulting consistent orthonormal family.
 Moreover, we apply the Adam algorithm over 5000 epochs with learning rate $5 \cdot 10^{-5}$ and batchsize $100$ to train the Proximal Equilibrium Operator on a training set consisting of 200 European call options $H_k \eqdef \max(S_T-K_k,0)$, $k = 1,...,200$, and 200 European put options $H_k \eqdef \max(K_k-S_T,0)$, $k = 201,...,400$. In addition, we evaluate its generalization performance every 125-th epoch on a test set consisting of 50 European gap call options $H_k \eqdef (S_T-K_{k,1}) \mathds{1}_{\lbrace S_T \geq K_{k,2} \rbrace}$, $k = 401,...,450$, and 50 European gap put options $H_k \eqdef (K_{k,2}-S_T) \mathds{1}_{\lbrace S_T \leq K_{k,1} \rbrace}$, $k = 451,...,500$. Hereby, the parameters $K_1,...,K_{400} \geq 0$ and $K_{401,1},K_{401,2},...,K_{500,1},K_{500,2} \geq 0$ are randomly initialized with $K_{k,1} \leq K_{k,2}$, whereas the optimal prices and hedging strategies $(x_k,\theta_k) = \mathcal{S}(g_{H_k})$ are computed with the help of the minimal equivalent local martingale measure (see \cite{Schweizer_1999_QuadraticHedging,Pham_2000_QuadraticHedging}) and the Fourier arguments in \cite{Carr_1999_OptionFourier}. The results are reported in Figure~\ref{FigQuadHedg}.
\end{example}

\begin{figure}[h]
    \begin{minipage}[t]{0.49\textwidth}
        \includegraphics[width=0.95\linewidth]{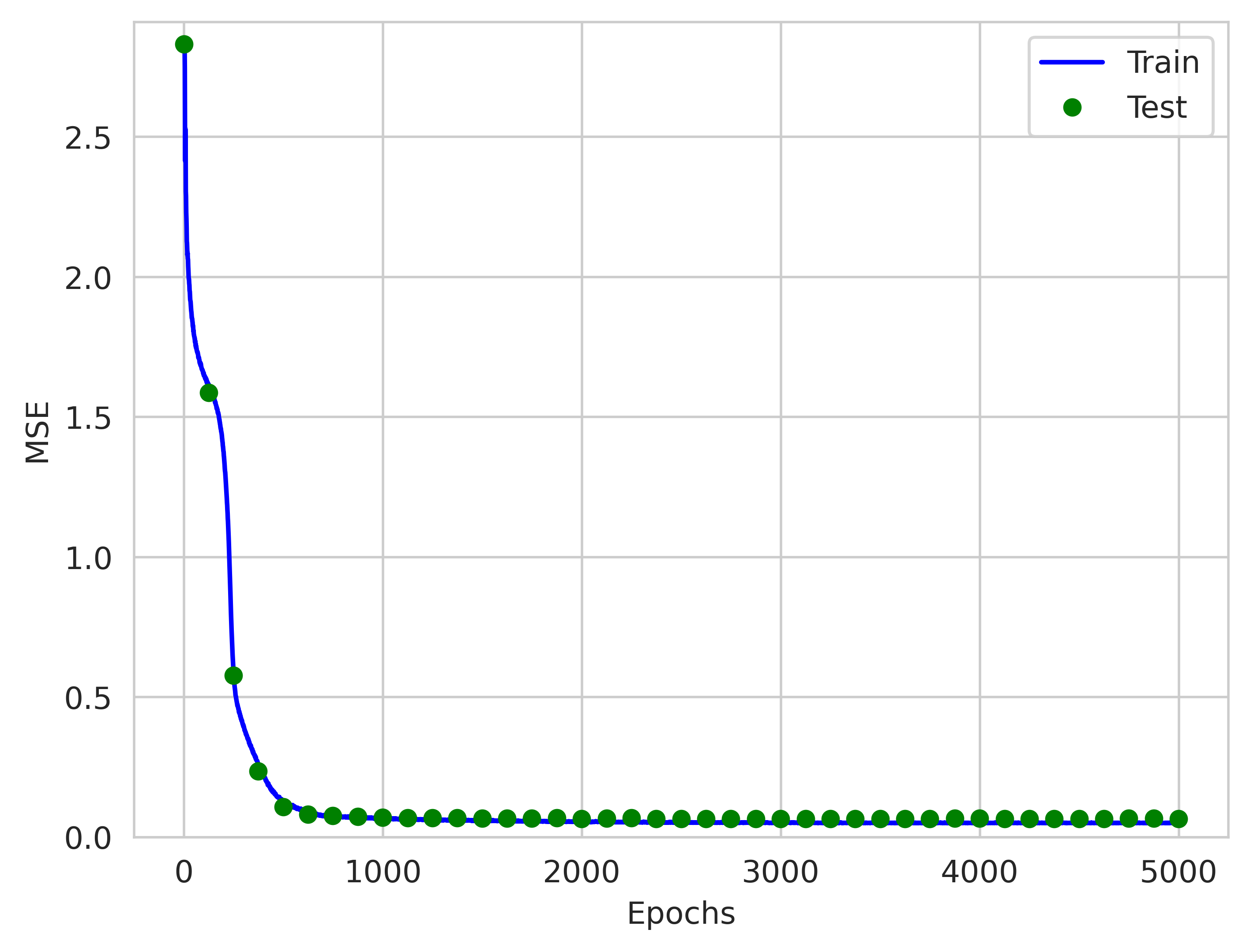}

        \subcaption{Learning performance}
    \end{minipage}
    \begin{minipage}[t]{0.49\textwidth}
        \includegraphics[width=1.0\linewidth]{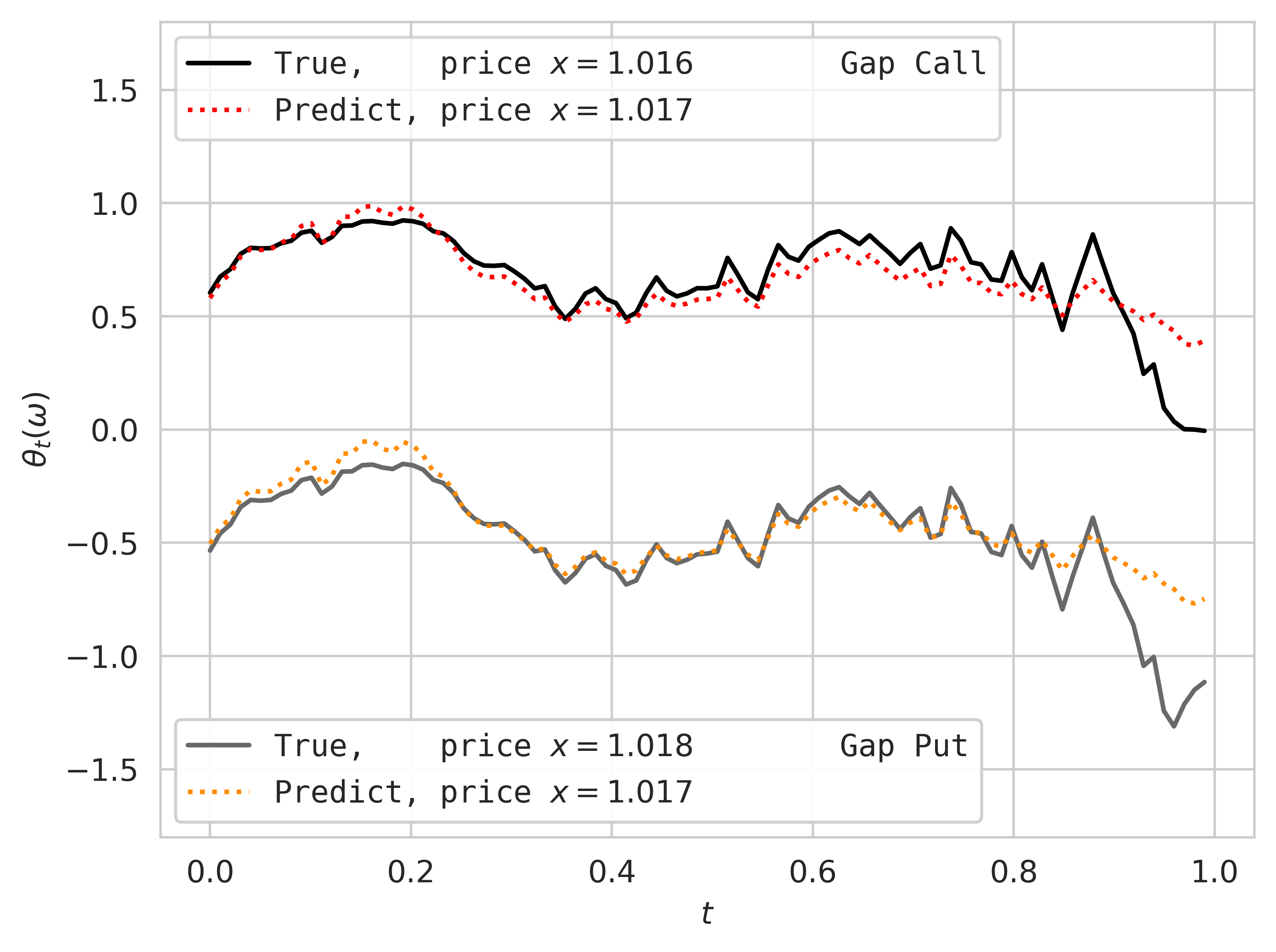}

        \subcaption{Optimal tuple $(x,\theta) \in X$ for one gap call option and one gap put option of the test set and some fixed $\omega \in \Omega$}        
    \end{minipage}
    \caption{Learning the pricing/hedging operator $\mathcal{S}$ in \eqref{EqQuadHedg} by a Proximal Equilibrium Operator~$\mathcal{G}$. In (a), the learning performance is displayed in terms of the mean squared error (MSE) $\frac{1}{\vert K \vert} \sum_{k \in K} \Vert \mathcal{S}(g_{H_k}) - \mathcal{G}(g_{H_k}) \Vert^2$ on the training set (label ``Train'') and test set (label ``Test''). In (b), the predicted solution $\mathcal{G}(g_{H_k})$ (label ``Predict'') is compared to the true solution $\mathcal{S}(g_{H_k})$ (label ``True'') for one $k$ of the test set and some $\omega \in \Omega$.}
    \label{FigQuadHedg}
\end{figure}

\section{Conclusion}
\label{s:Conclusion}
This paper helps close the gap between neural operator theory which suggests that solving infinite-dimensional problems requires exponentially large models, 
whereas reasonably sized neural operators have consistently succeeded in experimental practice. We address this gap by designing a proximal neural operator, the PEO model, whose internal architecture efficiently encodes proximal forward-backward splitting algorithms at scale. Our main results (Theorems~\ref{thrm:Main_AproximateSelection} and~\ref{thrm:Main_NearOptimization}) demonstrate that this architecture does not suffer from the theory–practice gap: it can uniformly approximate the (approximate) solution operator for infinite families of convex splitting problems of the form~\eqref{eq:optim}, to arbitrary accuracy, with complexity that scales only logarithmically in the residual approximation error.

To illustrate the broad scope of our results, we show that solution operators for a broad class of problems, including parametric non-linear PDEs (Section~\ref{s:Applications__ss:PDEs}), stochastic optimal control and portfolio optimization (Section~\ref{s:Applications__ss:Control}), and quadratic hedging under liquidity constraints (Section~\ref{s:Applications__ss:Finance}), can all be cast in the form~\eqref{eq:optim}, and are therefore tractable using small PEOs. Each of these theoretical claims is also validated through empirical experiments.

\section{Proofs of the Solver Guarantees}
\label{s:Proofs}

\subsection{Proximal Calculus and Smooth Convex Losses}

\begin{lemma}[Scaled Proximal Calculus]
\label{lem:scaled_proximal_calculus}
For $f\in\Gamma_0(X)$ and $h>0$, \eqref{eq:prox_f__definition} has exactly one minimizer for every $u\in X$. Characterization \eqref{eq:proximal_equivalence} holds and
\begin{equation}
\label{eq:proximal_firm_bound}
 \|\operatorname{prox}_{hf}(u)-\operatorname{prox}_{hf}(v)\|_X^2
 \le\langle u-v,\operatorname{prox}_{hf}(u)-\operatorname{prox}_{hf}(v)\rangle_X.
\end{equation}
If $f-\frac\mu2\|\cdot\|_X^2$ is convex for $\mu>0$, then $\operatorname{prox}_{hf}$ is $(1+h\mu)^{-1}$-Lipschitz.
\end{lemma}
\begin{proof}
\textit{Existence.} Choose $a\in\operatorname{dom}(f)$. Lower semicontinuity at $a$ supplies $s>0$ such that $f(w)\ge f(a)-1$ whenever $\|w-a\|_X<s$. For $y\in\operatorname{dom}(f)$ with $d\eqdef\|y-a\|_X\ge s/2$, set $t=s/(2d)\in(0,1]$. Convexity and the bound at $a+t(y-a)$ give
\[
 f(a)-1\le f(a+t(y-a))\le(1-t)f(a)+tf(y),
 \qquad f(y)\ge f(a)-2d/s.
\]
For $d<s/2$, the local lower bound gives $f(y)\ge f(a)-1$. For $y\notin\operatorname{dom}(f)$ every finite lower bound is automatic. Thus, for all $y\in X$,
\[
 f(y)\ge f(a)-1-(2/s)\|y-a\|_X.
\]
Fix $u\in X$ and write $H_u(y)\eqdef hf(y)+\frac12\|y-u\|_X^2$. The preceding inequality and $\|y-a\|_X\le\|y-u\|_X+\|u-a\|_X$ show that $H_u$ is bounded below. Since $H_u(a)<\infty$, its infimum $m$ is real. Choose $y_n\in\operatorname{dom}(f)$ with $H_u(y_n)\to m$. Convexity of $f$ and the parallelogram identity yield
\[
 m\le H_u((y_n+y_k)/2)
 \le\frac{H_u(y_n)+H_u(y_k)}2-\frac18\|y_n-y_k\|_X^2.
\]
Consequently, $\|y_n-y_k\|_X^2\le4(H_u(y_n)-m)+4(H_u(y_k)-m)$, so $(y_n)$ is Cauchy. Completeness gives a limit $p\in X$. Lower semicontinuity gives $H_u(p)\le\liminf_nH_u(y_n)=m$, and therefore $H_u(p)=m$. The same midpoint inequality shows that two minimizers must coincide.

\textit{Optimality and nonexpansiveness.} Let $p$ be that minimizer. For $y\in\operatorname{dom}(f)$ and $t\in(0,1]$, compare $H_u(p)$ with $H_u(p+t(y-p))$ and use convexity of $f$. Expanding the squared norm gives
\[
 0\le h(f(y)-f(p))+\langle p-u,y-p\rangle_X+\frac t2\|y-p\|_X^2.
\]
Letting $t\downarrow0$ proves $(u-p)/h\in\partial f(p)$. Conversely, if this inclusion holds, expansion of the squared norm gives $H_u(y)\ge H_u(p)+\frac12\|y-p\|_X^2$ for every $y$; hence $p$ is the unique minimizer. This proves both directions of \eqref{eq:proximal_equivalence}.

Let $p=\operatorname{prox}_{hf}(u)$ and $z=\operatorname{prox}_{hf}(v)$. Apply the subgradient inequality at $p$ to $z$ and at $z$ to $p$, and add. The result is
\[
 0\ge\langle u-p,z-p\rangle_X+\langle v-z,p-z\rangle_X
 =\|p-z\|_X^2-\langle u-v,p-z\rangle_X.
\]
This is \eqref{eq:proximal_firm_bound}. Cauchy--Schwarz gives $\|p-z\|_X\le\|u-v\|_X$ when $p\ne z$, and the same inequality is immediate when $p=z$.

\textit{Strong convexity.} Suppose $f-\frac\mu2\|\cdot\|_X^2$ is convex. Its convexity inequality is
\[
 f(p+t(y-p))\le(1-t)f(p)+tf(y)-\frac\mu2t(1-t)\|y-p\|_X^2.
\]
Use this inequality instead of ordinary convexity in the preceding optimality comparison, divide by $t$, and let $t\downarrow0$. One obtains
\[
 hf(y)\ge hf(p)+\langle u-p,y-p\rangle_X+\frac{h\mu}{2}\|y-p\|_X^2.
\]
Apply this inequality to $p,z$ in both orders and add. It gives $(1+h\mu)\|p-z\|_X^2\le\langle u-v,p-z\rangle_X$. Cauchy--Schwarz and the separate case $p=z$ prove the claimed Lipschitz constant.
\end{proof}

\begin{lemma}[Smooth Convex Gradient Bounds]
\label{lem:smooth_convex_bounds}
Let $g\in\mathcal X_\lambda$. For $x,y\in X$,
\begin{align}
 g(y)&\ge g(x)+\langle\nabla g(x),y-x\rangle_X,\label{eq:convex_first_order}\\
 g(y)&\le g(x)+\langle\nabla g(x),y-x\rangle_X+\frac\lambda2\|y-x\|_X^2,\label{eq:smooth_descent}\\
 \langle\nabla g(x)-\nabla g(y),x-y\rangle_X
 &\ge\lambda^{-1}\|\nabla g(x)-\nabla g(y)\|_X^2.\label{eq:cocoercivity}
\end{align}
Consequently, $\operatorname{id}_X-h\nabla g$ is nonexpansive for $0<h\le2/\lambda$.
\end{lemma}
\begin{proof}
Convexity gives $[g(x+t(y-x))-g(x)]/t\le g(y)-g(x)$ for $0<t\le1$. Taking $t\downarrow0$ proves \eqref{eq:convex_first_order}. The scalar fundamental theorem of calculus along the segment gives
\[
 g(y)-g(x)-\langle\nabla g(x),y-x\rangle_X
 =\int_0^1\langle\nabla g(x+t(y-x))-\nabla g(x),y-x\rangle_X\,dt.
\]
The integrand is continuous and is bounded above by $\lambda t\|y-x\|_X^2$. Integration proves \eqref{eq:smooth_descent}.

For fixed $y$, put $\phi(v)=g(v)-\langle\nabla g(y),v\rangle_X$. By \eqref{eq:convex_first_order}, $\phi(y)\le\phi(v)$ for every $v$. Let $d=\nabla g(x)-\nabla g(y)$ and $w=x-d/\lambda$. Subtracting the linear term from \eqref{eq:smooth_descent} and substituting $w$ gives
\[
 \phi(y)\le\phi(w)\le\phi(x)+\langle d,-d/\lambda\rangle_X
 +\frac\lambda2\|d/\lambda\|_X^2
 =\phi(x)-\frac1{2\lambda}\|d\|_X^2.
\]
Hence $g(x)-g(y)-\langle\nabla g(y),x-y\rangle_X\ge\|d\|_X^2/(2\lambda)$. Repeating the displayed calculation with $x$ and $y$ interchanged and adding the two resulting inequalities proves \eqref{eq:cocoercivity}.

Finally, writing $a=x-y$ and using \eqref{eq:cocoercivity},
\[
 \|a-hd\|_X^2=\|a\|_X^2-2h\langle a,d\rangle_X+h^2\|d\|_X^2
 \le\|a\|_X^2-(2h/\lambda-h^2)\|d\|_X^2\le\|a\|_X^2.
\]
The last coefficient is nonnegative for $0<h\le2/\lambda$. Taking square roots proves nonexpansiveness.
\end{proof}
The cocoercivity assertion is the Baillon--Haddad theorem; see \cite[Theorem~1.1, p.~782]{BauschkeCombettes2010BaillonHaddad}. Here the source's real Hilbert space, convex real-valued function, and positive global gradient-Lipschitz bound correspond to $X$, $g$, and $\lambda$. The proof above establishes this version directly.

\begin{lemma}[Uniform Gradient and Difference Bounds]
\label{lem:SimpleGradBound}
Let $g\in\mathcal X_\lambda(r)$. Then $\|\nabla g(x)\|_X\le\sqrt r$ for every $x$, and $g$ is $\sqrt r$-Lipschitz. For $R\in\mathbb N_+$ and $\delta>0$, define
\begin{equation}
\label{eq:finite_difference}
 \Delta_\delta^R(g)(x)\eqdef\sum_{j=0}^{R-1}\frac{g(x+\delta e_j)-g(x)}\delta e_j.
\end{equation}
Then

\begin{equation}
\label{eq:lem:SimpleGradBound}
 \sup_{x\in X}\|\Delta_\delta^R(g)(x)-\nabla g(x)\|_X
 \le\frac{\lambda\delta\sqrt R}{2}+\sqrt r\,2^{-R/4}.
\end{equation}

\end{lemma}
\begin{proof}
Taking $R=0$ in \eqref{eq:nice_gs} and applying Parseval's identity gives $\|\nabla g(x)\|_X^2\le r$. For $x,y\in X$, the scalar fundamental theorem of calculus therefore gives
\[
 |g(y)-g(x)|=\left|\int_0^1\langle\nabla g(x+t(y-x)),y-x\rangle_X\,dt\right|
 \le\sqrt r\|y-x\|_X.
\]
Fix $j<R$. Applying the same theorem on the segment $x+te_j$, $0\le t\le\delta$, yields
\[
 \frac{g(x+\delta e_j)-g(x)}\delta-\langle\nabla g(x),e_j\rangle_X
 =\frac1\delta\int_0^\delta\langle\nabla g(x+te_j)-\nabla g(x),e_j\rangle_X\,dt.
\]
The absolute value is at most $\delta^{-1}\int_0^\delta\lambda t\,dt=\lambda\delta/2$. Orthonormality now gives
\[
 \|\Delta_\delta^R(g)(x)-P_R\nabla g(x)\|_X^2
 \le R\lambda^2\delta^2/4.
\]
Also, \eqref{eq:nice_gs} gives $\|(I-P_R)\nabla g(x)\|_X\le\sqrt r2^{-R/4}$. The triangle inequality, followed by the supremum over $x$, proves \eqref{eq:lem:SimpleGradBound}. Every bound is independent of $g\in\mathcal X_\lambda(r)$.
\end{proof}

\subsection{Exact Splitting and the Selector}

\begin{lemma}[Exact Proximal-Gradient Objective Rate]
\label{lem:Splitting_Convergence}
Let $f\in\Gamma_0(X)$, $g\in\mathcal X_\lambda$, $h\in(0,1/\lambda]$, and $x^\star\in\operatorname*{argmin}\ell_{f,g}$. The map $T_g=\operatorname{prox}_{hf}\circ(\operatorname{id}_X-h\nabla g)$ is nonexpansive, and $\operatorname{Fix}(T_g)=\operatorname*{argmin}\ell_{f,g}$. The trajectory $x_0^g=0$, $x_{l+1}^g=T_gx_l^g$ satisfies

\begin{equation}
\label{eq:Linear_convergence_Guo}
 0\le\ell_{f,g}(x_L^g)-\ell_{f,g}(x^\star)\le\frac{\|x^\star\|_X^2}{2hL},
 \quad L\in\mathbb N_+.
\end{equation}

Moreover, $\|x_l^g-x^\star\|_X\le\|x^\star\|_X$ for every $l\in\mathbb N$.
\end{lemma}
\begin{proof}
Lemma~\ref{lem:scaled_proximal_calculus} and Lemma~\ref{lem:smooth_convex_bounds} give nonexpansiveness of the two factors of $T_g$, since $h\le1/\lambda<2/\lambda$.

We verify the fixed-point assertion without an unproved subdifferential sum rule. If $p=T_gp$, then $-\nabla g(p)\in\partial f(p)$ by Lemma~\ref{lem:scaled_proximal_calculus}. Add its subgradient inequality to \eqref{eq:convex_first_order} to obtain $f(y)+g(y)\ge f(p)+g(p)$ for every $y$. Conversely, suppose $p$ minimizes $f+g$. Its value is finite because $f$ is proper and $g$ is real-valued. For $y\in\operatorname{dom}(f)$ and $0<t\le1$, minimality and convexity of $f$ give
\[
 f(p)+g(p)\le f(p+t(y-p))+g(p+t(y-p))
 \le(1-t)f(p)+tf(y)+g(p+t(y-p)).
\]
Subtract $f(p)+g(p)$, divide by $t$, and take $t\downarrow0$. This gives $f(y)-f(p)+\langle\nabla g(p),y-p\rangle_X\ge0$. The same inequality is automatic when $f(y)=\infty$. Thus $-\nabla g(p)\in\partial f(p)$, and Lemma~\ref{lem:scaled_proximal_calculus} gives $p=T_gp$.

For $x\in X$, put $p=T_gx$. The proximal optimality condition gives $(x-p)/h-\nabla g(x)\in\partial f(p)$, so for $y\in\operatorname{dom}(f)$,
\[
 f(p)-f(y)\le h^{-1}\langle x-p,p-y\rangle_X-\langle\nabla g(x),p-y\rangle_X.
\]
By \eqref{eq:smooth_descent} and \eqref{eq:convex_first_order},
\[
 g(p)-g(y)\le\langle\nabla g(x),p-y\rangle_X+\frac\lambda2\|p-x\|_X^2.
\]
Adding and using the identity $2\langle x-p,p-y\rangle_X=\|x-y\|_X^2-\|p-y\|_X^2-\|p-x\|_X^2$ proves
\begin{equation}
\label{eq:one_step_objective_bound}
 \ell_{f,g}(p)-\ell_{f,g}(y)
 \le\frac{\|x-y\|_X^2-\|p-y\|_X^2}{2h}
 -\frac{1-h\lambda}{2h}\|p-x\|_X^2.
\end{equation}
Take $x=x_l^g$ and $y=x^\star$. The left side is nonnegative, and $1-h\lambda\ge0$. This proves $\|x_{l+1}^g-x^\star\|_X\le\|x_l^g-x^\star\|_X$. Induction from $x_0^g=0$ gives the stated norm bound.

Each $x_l^g$ with $l\ge1$ belongs to $\operatorname{dom}(f)$ by Lemma~\ref{lem:scaled_proximal_calculus}. Taking $x=y=x_l^g$ in \eqref{eq:one_step_objective_bound} therefore proves $\ell_{f,g}(x_{l+1}^g)\le\ell_{f,g}(x_l^g)$ for $l\ge1$. No assumption $0\in\operatorname{dom}(f)$ is used. Summing \eqref{eq:one_step_objective_bound} with $x=x_l^g$, $y=x^\star$, for $l=0,\ldots,L-1$, and discarding the nonpositive squared-increment terms yields
\[
 \sum_{l=0}^{L-1}\bigl[\ell_{f,g}(x_{l+1}^g)-\ell_{f,g}(x^\star)\bigr]
 \le\frac{\|x^\star\|_X^2-\|x_L^g-x^\star\|_X^2}{2h}
 \le\frac{\|x^\star\|_X^2}{2h}.
\]
The objective values from index $1$ onward are nonincreasing; hence the sum is at least $L[\ell_{f,g}(x_L^g)-\ell_{f,g}(x^\star)]$. Division by $L>0$ proves \eqref{eq:Linear_convergence_Guo}.
\end{proof}

\begin{proof}[Proof of Proposition~\ref{prop:Existence}]
Fix $g\in\mathfrak G$ and a minimizer $x_g^\star$ supplied by Assumption~\ref{ass:uniform_splitting}(i). Lemma~\ref{lem:Splitting_Convergence} gives
\[
 \ell_{f,g}(x_{L_\eta}^g)-\inf_X\ell_{f,g}
 \le D^2/(2hL_\eta)\le\eta.
\]
This includes $D=0$, because then the numerator is zero and $L_\eta=1$.

For continuity, fix this $g$ and put $K_g=\{x_0^g,\ldots,x_{L_\eta-1}^g\}$. This finite set is compact. For $\widetilde g\in\mathfrak G$, nonexpansiveness of $T_{\widetilde g}$ and of the proximal map gives
\begin{align*}
 \|x_{l+1}^{\widetilde g}-x_{l+1}^g\|_X
 &\le\|T_{\widetilde g}x_l^{\widetilde g}-T_{\widetilde g}x_l^g\|_X
       +\|T_{\widetilde g}x_l^g-T_gx_l^g\|_X\\
 &\le\|x_l^{\widetilde g}-x_l^g\|_X+h\|\nabla\widetilde g(x_l^g)-\nabla g(x_l^g)\|_X\\
 &\le\|x_l^{\widetilde g}-x_l^g\|_X+hp_{K_g}(\widetilde g-g).
\end{align*}
The two initial states are zero. Induction on $l=0,\ldots,L_\eta$ gives $\|x_l^{\widetilde g}-x_l^g\|_X\le hlp_{K_g}(\widetilde g-g)$. In particular, for every $a>0$ the neighborhood $\{\widetilde g:p_{K_g}(\widetilde g-g)<a/(hL_\eta)\}$ maps into the radius-$a$ ball about $S_\eta(g)$. This proves continuity in the original topology $\tau$.
\end{proof}

\subsection{Realization and Stable Perturbation}

\begin{lemma}[Exact Finite-Difference Layer Realization]
\label{lem:finite_difference_layer}
Let $f\in\Gamma_0(X)$, $R,L\in\mathbb N_+$, and $\delta,h>0$. Define
$
 D_{R,\delta}\eqdef\delta^{-1}\begin{bmatrix}-\mathbf1_R&I_R\end{bmatrix}$, $\mathbf v_g(z)\eqdef\bigl(g(z),g(z+\delta e_0),\ldots,g(z+\delta e_{R-1})\bigr)^\top$.  
The choices
\begin{equation}
\label{eq:constructive_weights}
 \begin{gathered}
 M=R+1,\quad A^{(l)}=I_R,\quad B^{(l)}=-hD_{R,\delta},
 \quad b^{(l)}=0,\quad\gamma^{(l)}=0
 ,\quad x_1^{(l)}=0,\quad x_{j+2}^{(l)}=\delta e_j
 \end{gathered}
\end{equation}
for $j=0,\dots,R-1$; 
realize the recursion
\begin{equation}
\label{eq:FBSplitting_approximate}
 z_0^g=0,\qquad
 z_{l+1}^g=P_R\operatorname{prox}_{hf}\bigl(z_l^g-h\Delta_\delta^R(g)(z_l^g)\bigr),\quad0\le l<L.
\end{equation}
The resulting operator returns $z_L^g$ and is continuous on $C(X;\mathbb R)$ with its compact-open topology. A direct evaluation uses $L(R+1)$ loss-oracle calls, $L$ projected-proximal calls, $3LR$ nonzero affine coefficients, and $LR$ nonzero sampling-shift coordinates.
\end{lemma}
\begin{proof}
For each $j=0,\ldots,R-1$, coordinate $j$ of $D_{R,\delta}\mathbf v_g(z)$ is $(g(z+\delta e_j)-g(z))/\delta$. Applying synthesis gives $(D_{R,\delta}\mathbf v_g(z))^{\uparrow:R}=\Delta_\delta^R(g)(z)$. Since $(z^{\downarrow:R})^{\uparrow:R}=z$ for $z\in E_R$, substitution of \eqref{eq:constructive_weights} into \eqref{eq:corrected_architecture} gives \eqref{eq:FBSplitting_approximate}. The sign is negative, and the first column of $D_{R,\delta}$ supplies the common baseline evaluation. Each argument of the proximal oracle lies in $E_R$, and its projected output also lies in $E_R$. Induction therefore proves well-posedness through the $L$ layers.

For continuity, first consider $(g,z)\mapsto g(z+s)$ on $C(X;\mathbb R)\times E_R$ for fixed $s\in E_R$. At a point $(g_0,z_0)$, the set $K=\{z+s:z\in E_R,\|z-z_0\|_X\le1\}$ is compact: its coordinates form a closed bounded subset of the finite-dimensional space $\mathbb R^R$. For $\|z-z_0\|_X<1$,
\[
 |g(z+s)-g_0(z_0+s)|\le\sup_{w\in K}|g(w)-g_0(w)|+|g_0(z+s)-g_0(z_0+s)|.
\]
The first term is controlled by a compact-open neighborhood of $g_0$, and the second tends to zero as $z\to z_0$ by continuity of $g_0$. This proves joint continuity for this finite-dimensional state space. Lemma~\ref{lem:scaled_proximal_calculus} implies continuity of the activation, and all coordinate affine maps are continuous. Induction over the finitely many layers proves continuity of $g\mapsto z_L^g$.

Each layer evaluates the $R+1$ displayed loss values and calls the projected-proximal oracle once. Its matrix $A^{(l)}$ has $R$ nonzeros and its matrix $B^{(l)}$ has $2R$ nonzeros. The bias and gate are zero. Among its sampling shifts there are exactly $R$ nonzero coordinates. Multiplication by $L$ proves the counts, with repeated copies counted separately. No parameters internal to an oracle are included.
\end{proof}

\begin{lemma}[Uniform Perturbed-Trajectory Control]
\label{lem:Discretization_bound}
Under Assumption~\ref{ass:uniform_splitting}(i)--(ii), fix $L,R\in\mathbb N_+$, $\delta>0$, and $q\in(0,1]$ such that every $T_g$ is $q$-Lipschitz. Define
\begin{equation}
\label{eq:perturbation_parameters}
 a_{R,\delta}\eqdef\frac{\lambda\delta\sqrt R}{2}+\sqrt r2^{-R/4},\qquad
 d_{R,\delta}\eqdef ha_{R,\delta}+C_f2^{-R},\qquad
 W_n(q)\eqdef\sum_{j=0}^{n-1}q^j,
\end{equation}
with $W_0(q)=0$. If $W_L(q)d_{R,\delta}\le1$, then the trajectories \eqref{eq:FB_Splitting} and \eqref{eq:FBSplitting_approximate} satisfy
\begin{equation}
\label{eq:FBSplitting_approximate__Estimate0}
 \sup_{g\in\mathfrak G}\|x_n^g-z_n^g\|_X\le W_n(q)d_{R,\delta},
 \qquad0\le n\le L.
\end{equation}
Moreover, $\|x_n^g\|_X\le2D$, $\|z_n^g\|_X\le2D+1$, and every proximal input in the approximate trajectory has norm at most $B$.
\end{lemma}
\begin{proof}
Fix $g\in\mathfrak G$ and a minimizer from Assumption~\ref{ass:uniform_splitting}(i). Lemma~\ref{lem:Splitting_Convergence} gives
\[
 \|x_n^g\|_X\le\|x_n^g-x_g^\star\|_X+\|x_g^\star\|_X\le2D.
\]
We prove by induction on $n$ that $\|x_n^g-z_n^g\|_X\le W_n(q)d_{R,\delta}$. It holds for $n=0$, because both states vanish. Suppose it holds at $n<L$. Since $W_n(q)\le W_L(q)$, the induction hypothesis gives $\|z_n^g\|_X\le2D+1$. Lemma~\ref{lem:SimpleGradBound} gives both $\|\nabla g(z_n^g)\|_X\le\sqrt r$ and $\|\Delta_\delta^R(g)(z_n^g)-\nabla g(z_n^g)\|_X\le a_{R,\delta}$. Put $u_n=z_n^g-h\Delta_\delta^R(g)(z_n^g)\in E_R$. Then
\[
 \|u_n\|_X\le2D+1+h\sqrt r+ha_{R,\delta}\le2D+2+h\sqrt r=B.
\]
For the second inequality, $ha_{R,\delta}\le d_{R,\delta}\le1$ because $W_L(q)\ge1$ and $W_L(q)d_{R,\delta}\le1$. Thus Assumption~\ref{ass:uniform_splitting}(ii) applies at $u_n$; its applicability was not assumed as an unproved trajectory invariant.

Using the $q$-Lipschitz bound for $T_g$, the $1$-Lipschitz bound for the proximal map, Lemma~\ref{lem:SimpleGradBound}, and \eqref{eq:proximal_tail_assumption} in that order gives
\begin{align*}
 \|x_{n+1}^g-z_{n+1}^g\|_X
 &\le\|T_gx_n^g-T_gz_n^g\|_X
 +\|T_gz_n^g-\operatorname{prox}_{hf}(u_n)\|_X
 +\|(I-P_R)\operatorname{prox}_{hf}(u_n)\|_X\\
 &\le q\|x_n^g-z_n^g\|_X
 +h\|\nabla g(z_n^g)-\Delta_\delta^R(g)(z_n^g)\|_X+C_f2^{-R}\\
 &\le qW_n(q)d_{R,\delta}+d_{R,\delta}
 =W_{n+1}(q)d_{R,\delta}.
\end{align*}
Here $W_{n+1}(q)=1+qW_n(q)$ by its finite-sum definition. This completes the induction. It also gives the bound on $z_L^g$; the input bound was proved at every index $n<L$. All constants are independent of the fixed $g$, so taking the supremum proves every asserted uniform bound.
\end{proof}

\subsection{Main Realization and Objective Bounds}

\begin{proof}[Proof of Theorem~\ref{thrm:Main_AproximateSelection}]
The choices in the statement of Theorem~\ref{thrm:Main_AproximateSelection} are admissible: $L,R,M$ are positive integers and $\delta>0$, because $\eta,\varepsilon,h,\lambda>0$. We apply Lemma~\ref{lem:Discretization_bound} with $q=1$, which is permitted by Lemma~\ref{lem:Splitting_Convergence}. First, verify its small-error hypothesis. The choice of $\delta$ gives
\[
 \frac{h\lambda\delta\sqrt R}{2}=\frac{\varepsilon}{3L}.
\]
For every $a\ge0$, $a/(1+3La/\varepsilon)\le\varepsilon/(3L)$; multiplying by the positive denominator reduces this to $a\le a+\varepsilon/(3L)$. The two lower bounds built into $R$ therefore give
\[
 h\sqrt r2^{-R/4}\le\frac{h\sqrt r}{1+3Lh\sqrt r/\varepsilon}\le\frac{\varepsilon}{3L},\qquad
 C_f2^{-R}\le\frac{C_f}{1+3LC_f/\varepsilon}\le\frac{\varepsilon}{3L}.
\]
These inequalities remain valid when $r=0$ or $C_f=0$. Hence $Ld_{R,\delta}\le\varepsilon\le1$. Since $W_L(1)=L$, Lemma~\ref{lem:Discretization_bound} gives
\[
 \sup_{g\in\mathfrak G}\|x_L^g-z_L^g\|_X\le\varepsilon.
\]
Lemma~\ref{lem:finite_difference_layer} realizes $z_L^g$ by the explicitly specified operator with coefficients \eqref{eq:constructive_weights}. Because $L=L_\eta$ and $S_\eta(g)=x_L^g$, this is \eqref{eq:thrm:Main_AproximateSelection__Approx}. The same realization lemma gives all oracle and coefficient counts.

For completeness, $L\le1+D^2/(2h\eta)$. Set $C_0=\max\{1,3h\sqrt r,3C_f\}$. Since $1+ct\le\max\{1,c\}(1+t)$ for $c,t\ge0$, the ceiling bound $\lceil u\rceil\le u+1$ gives
\[
 R\le1+4\log_2C_0+4\log_2(1+L/\varepsilon),\qquad M=R+1.
\]
These explicit bounds imply the asserted asymptotics with only the stated structural dependencies.
\end{proof}

\begin{proof}[Proof of Theorem~\ref{thrm:Main_NearOptimization}]
Fix $g\in\mathfrak G$ and abbreviate $x=S_\eta(g)$ and $z=\mathcal G_{\eta,\varepsilon}(g)$. Proposition~\ref{prop:Existence} and the completed proof of Theorem~\ref{thrm:Main_AproximateSelection} give $\ell_{f,g}(x)-\inf_X\ell_{f,g}\le\eta$ and $\|x-z\|_X\le\varepsilon$. Lemma~\ref{lem:Discretization_bound} gives $\|x\|_X\le2D$ and $\|z\|_X\le2D+1$, so the assumed local Lipschitz bound for $f$ applies to these two states. Lemma~\ref{lem:SimpleGradBound} supplies the global $\sqrt r$-Lipschitz bound for $g$. Therefore
\begin{align*}
 \ell_{f,g}(z)-\inf_X\ell_{f,g}
 &=f(z)-f(x)+g(z)-g(x)+\ell_{f,g}(x)-\inf_X\ell_{f,g}\\
 &\le(\lambda_f+\sqrt r)\|z-x\|_X+\eta
 \le\eta+(\lambda_f+\sqrt r)\varepsilon.
\end{align*}
Every term is finite by the local finiteness hypothesis and the fact that $g$ is real-valued. Taking the supremum over $g$ proves the statement.
\end{proof}

\begin{proof}[Proof of Proposition~\ref{prop:feasible_objective}]
For $f=\iota_C$, definition \eqref{eq:prox_f__definition} says that $\operatorname{prox}_{hf}(u)$ is the unique closest point in $C$ to $u$; existence and uniqueness follow from Lemma~\ref{lem:scaled_proximal_calculus}. In particular, each exact iterate with index at least one lies in $C$. Each approximate iterate with index at least one is the orthogonal projection $P_R$ of a point of $C$, and therefore also belongs to $C$ by the hypothesis $P_R(C)\subseteq C$. Both terminal indices are at least one, even if $0\notin C$.

Put $x=S_\eta(g)$ and $z=\mathcal G_{\eta,\varepsilon}(g)$. Their feasibility gives $f(x)=f(z)=0$. By Proposition~\ref{prop:Existence}, $g(x)-\inf_Cg\le\eta$. Lemma~\ref{lem:SimpleGradBound} and Theorem~\ref{thrm:Main_AproximateSelection} now give
\[
 g(z)-\inf_Cg\le|g(z)-g(x)|+g(x)-\inf_Cg
 \le\sqrt r\|z-x\|_X+\eta\le\sqrt r\varepsilon+\eta.
\]
Taking the supremum proves the claim. No Lipschitz estimate for the extended-valued indicator is used.
\end{proof}

\subsection{The Separate Contractive Regime}

\begin{proof}[Proof of Proposition~\ref{prop:contractive_realization}]
For each $g$, Assumption~\ref{ass:uniform_splitting}(i) supplies a minimizer $x_g^\star$. Lemma~\ref{lem:Splitting_Convergence} makes it a fixed point. If $y_g^\star$ were another minimizer, the same lemma and \eqref{eq:uniform_contraction} would give $\|x_g^\star-y_g^\star\|_X\le q\|x_g^\star-y_g^\star\|_X$. Since $1-q>0$, the two points coincide.

For the exact trajectory, the bound $\|x_n^g-x_g^\star\|_X\le q^nD$ holds at $n=0$ by Assumption~\ref{ass:uniform_splitting}(i). If it holds at $n$, applying \eqref{eq:uniform_contraction} to $x_n^g,x_g^\star$ gives the bound $q^{n+1}D$ at $n+1$. This proves it for all $n$. The definition of $L$ in \eqref{eq:contractive_parameters} therefore gives
\[
 \|x_L^g-x_g^\star\|_X\le q^LD\le\frac{D}{1+2D/\varepsilon}\le\frac\varepsilon2.
\]
For $D=0$ the middle expression is zero; for $D>0$ the final inequality follows by multiplying by the positive denominator. All logarithms and denominators in \eqref{eq:contractive_parameters} are valid because $0<q<1$.

For the realization error, the chosen $\delta$ gives $h\lambda\delta\sqrt R/2=(1-q)\varepsilon/6$. Using the rank bounds and the elementary inequality $a/(1+6a/((1-q)\varepsilon))\le(1-q)\varepsilon/6$ for $a\ge0$, we obtain
\[
 h\sqrt r2^{-R/4}\le(1-q)\varepsilon/6,\qquad
 C_f2^{-R}\le(1-q)\varepsilon/6.
\]
Thus $d_{R,\delta}\le(1-q)\varepsilon/2$. The finite geometric sum satisfies $W_L(q)=(1-q^L)/(1-q)\le1/(1-q)$, so $W_L(q)d_{R,\delta}\le\varepsilon/2\le1$. All hypotheses of Lemma~\ref{lem:Discretization_bound} are now checked. Its conclusion gives
\[
 \|x_L^g-z_L^g\|_X\le\varepsilon/2.
\]
The triangle inequality yields $\|z_L^g-x_g^\star\|_X\le\varepsilon$, uniformly in $g$. Lemma~\ref{lem:finite_difference_layer} supplies the operator and the count $L(R+1)$.

To verify the rates, for $0<\varepsilon\le1$ use $1+c/\varepsilon\le(1+c)/\varepsilon$ when $c\ge0$, together with $\lceil u\rceil\le u+1$ in each of the three integer choices. This bounds $L$ and $R$ by a structural constant times $1+\log(1/\varepsilon)$, and $M=R+1$ obeys the same bound. Multiplication gives the stated quadratic-logarithmic total count. The constants may depend on $q$ and are not uniform as $q\uparrow1$.

Finally, if $f-\frac\mu2\|\cdot\|_X^2$ is convex, Lemma~\ref{lem:scaled_proximal_calculus} bounds the backward map by $(1+h\mu)^{-1}$, while Lemma~\ref{lem:smooth_convex_bounds} bounds the forward map by $1$. Their composition is therefore $(1+h\mu)^{-1}$-Lipschitz. This number lies strictly between $0$ and $1$, proving the sufficient condition.
\end{proof}

\begin{lemma}[Basis-Preserving Proximal Evaluation]
\label{lem:basis_preserving_proximal}
Let $f\in\Gamma_0(X)$ and suppose $f(P_Ry)\le f(y)$ for every $y\in X$ and every $R\in\mathbb N_+$. Then $\operatorname{prox}_{hf}(E_R)\subseteq E_R$ for every $h>0$. Consequently, Assumption~\ref{ass:uniform_splitting}(ii) holds with $C_f=0$. For $u\in E_R$, the proximal point is the minimizer of the finite-dimensional problem obtained by restricting \eqref{eq:prox_f__definition} to $E_R$.
\end{lemma}
\begin{proof}
Fix $h>0$, $R\ge1$, and $u\in E_R$. By Lemma~\ref{lem:scaled_proximal_calculus}, let $p=\operatorname{prox}_{hf}(u)$. Orthogonality and the hypothesis give
\[
 hf(P_Rp)+\frac12\|P_Rp-u\|_X^2
 \le hf(p)+\frac12\|p-u\|_X^2-\frac12\|(I-P_R)p\|_X^2.
\]
The left side cannot be less than the minimum attained at $p$. Thus $\|(I-P_R)p\|_X^2=0$, so $p\in E_R$. Uniqueness of the full-space minimizer makes it the unique minimizer of the restricted problem. Its tail vanishes for every admissible input, proving the assertion about $C_f$. This reduces evaluation to a finite-dimensional convex minimization; it does not bound the time needed to solve that minimization.
\end{proof}

\begin{lemma}[Exponential Orthogonal Projection Error]
\label{lem:ProjFixed}
For $C\ge0$ and $a>0$, let $\mathcal K_{C,a}\eqdef\{x\in X:|\langle x,e_j\rangle_X|\le Ce^{-aj}\text{ for every }j\in\mathbb N\}$. Then $P_R(\mathcal K_{C,a})\subseteq\mathcal K_{C,a}$ and
\[
 \sup_{x\in\mathcal K_{C,a}}\|(I-P_R)x\|_X
 \le\frac{Ce^{-aR}}{\sqrt{1-e^{-2a}}},\qquad R\in\mathbb N_+.
\]
\end{lemma}
\begin{proof}
Projection retains the coordinates with $j<R$ and sets the other coordinates to zero, proving the inclusion. Parseval's identity and the geometric-series formula give
\[
 \|(I-P_R)x\|_X^2=\sum_{j=R}^{\infty}|\langle x,e_j\rangle_X|^2
 \le C^2\sum_{j=R}^{\infty}e^{-2aj}=\frac{C^2e^{-2aR}}{1-e^{-2a}}.
\]
Since $a>0$, the denominator is positive. Taking square roots and the supremum proves the estimate.
\end{proof}

\begin{example}[A Compatible Infinite-Dimensional Family]
\label{ex:compatible_loss_family}
Fix $\mu,r>0$ and put
\[
 \mathcal B_r\eqdef\left\{b\in X:\sum_{j=R}^\infty|\langle b,e_j\rangle_X|^2\le r2^{-R/2}\text{ for all }R\in\mathbb N\right\}.
\]
Take $f(x)=\frac\mu2\|x\|_X^2$, $g_b(x)=-\langle b,x\rangle_X$, and $\mathfrak G=\{g_b:b\in\mathcal B_r\}$. Assumption~\ref{ass:uniform_splitting} holds for any $\lambda>0$, $h\in(0,1/\lambda]$, $D=\sqrt r/\mu$, and $C_f=0$. The common contraction factor is $q=(1+h\mu)^{-1}$.
\end{example}
\begin{proof}
Each $g_b$ is affine and has constant gradient $-b$, so it is convex and its gradient is $\lambda$-Lipschitz for any positive $\lambda$. The definition of $\mathcal B_r$ gives precisely \eqref{eq:nice_gs}. Completing the square yields
\[
 f(x)+g_b(x)=\frac\mu2\|x-b/\mu\|_X^2-\frac{\|b\|_X^2}{2\mu}.
\]
The unique minimizer is $b/\mu$, whose norm is at most $\sqrt r/\mu$ by the $R=0$ constraint. Completing the square in the proximal objective gives $\operatorname{prox}_{hf}(u)=u/(1+h\mu)$. It preserves $E_R$, so clause (ii) holds with $C_f=0$ on every ball. Finally,
\[
 T_{g_b}(x)=\frac{x+hb}{1+h\mu},\qquad
 \|T_{g_b}(x)-T_{g_b}(y)\|_X=\frac{\|x-y\|_X}{1+h\mu}.
\]
Thus the stated contraction bound holds uniformly in $b$. The family is not confined to a fixed finite-dimensional linear subspace: for every $j\in\mathbb N$, the vector $b_j=\sqrt r\,2^{-j/4}e_j$ belongs to $\mathcal B_r$, since its squared tail is either $r2^{-j/2}\le r2^{-R/2}$ when $R\le j$, or zero when $R>j$. These vectors are linearly independent.
\end{proof}

\appendix
\section{Supplementary Material}
\label{a:supplmat}
\subsection{Examples of Proximal Operators}
\label{a:Proximal_fin_dim}

For the examples below, the input and output of the activation are in $E_R$, and $\sigma_{f,h,R}(u)=P_R\operatorname{prox}_{hf}(u)$. In finite dimension use a fixed orthonormal basis with $1\le R\le d$.

\begin{example}[Quadratic Penalty]
For $f(x)=\frac12\|x\|_X^2$ and $h>0$, completion of the square gives
\[
 \frac h2\|v\|_X^2+\frac12\|v-u\|_X^2
 =\frac{1+h}{2}\left\|v-\frac{u}{1+h}\right\|_X^2
 +\frac{h}{2(1+h)}\|u\|_X^2.
\]
Thus $\operatorname{prox}_{hf}(u)=u/(1+h)$ and $\sigma_{f,h,R}(u)=u/(1+h)$ for $u\in E_R$.
\end{example}

\begin{example}[ReLU as a Proximal Map]
Let $X=\mathbb R^d$, $R=d$, and $f=\iota_{[0,\infty)^d}$. The proximal problem separates into $d$ scalar minimizations of $(v_i-u_i)^2$ over $v_i\ge0$, each uniquely minimized at $\max\{u_i,0\}$. Since $P_d=\operatorname{id}_{\mathbb R^d}$,
\[
 \operatorname{prox}_{hf}(u)=\sigma_{f,h,d}(u)
 =\operatorname{ReLU}_d(u),\qquad h>0.
\]
The rank is $d$, including when $d=1$.
\end{example}

\begin{example}[Sequence Soft Thresholding]
Let $X=\ell^2(\mathbb N)$ with
$\|x\|_X=(\sum_{i=0}^{\infty}|x_i|^2)^{1/2}$ and define
$f(x)=\sum_{i=0}^{\infty}|x_i|\in[0,\infty]$.
This functional is proper, convex, and lower semicontinuous: it vanishes at zero and is the supremum of its continuous convex finite partial sums. For $h>0$, set
\[
 s_h(a)=\begin{cases}a+h,&a<-h,\\0,&|a|\le h,\\a-h,&a>h.\end{cases}
\]
Minimizing $h|v|+(v-a)^2/2$ separately on $v<0$ and $v>0$, and checking the one-sided derivatives at zero, gives the unique minimizer $s_h(a)$. If $x\in\ell^2$, only finitely many $i$ satisfy $|x_i|>h$, since each such coordinate contributes more than $h^2$ to a convergent series. Thus $(s_h(x_i))_i$ has finite support and belongs to $\operatorname{dom}(f)$. Summing the scalar minimizing inequalities proves
\[
 \operatorname{prox}_{hf}(x)=(s_h(x_i))_{i=0}^{\infty}.
\]
Strict convexity of the quadratic distance gives uniqueness. For $u\in E_R$, the activation has the first $R$ coordinates $s_h(u_i)$ and all other coordinates zero.
\end{example}

\begin{example}
Let $X=L^2(\Omega,\mathcal A,\mu;\mathbb R)$, let $c_1<c_2$ be finite, and suppose
\[
 C=\{x\in X:c_1\le x\le c_2\quad\mu\text{-almost everywhere}\}\ne\varnothing.
\]
Set $f=\iota_C$. For $u\in\mathbb R$, write $\pi(u)=\min\{\max\{u,c_1\},c_2\}$. Choose $a\in C$. The scalar projection inequalities give $|\pi(x)-a|\le|x-a|$, and consequently $|\pi(x)|\le|x|+2|a|$ almost everywhere. Thus $\pi\circ x\in X$ and belongs to $C$.

The set $C$ is convex. It is closed: if $x_n\in C$ and $x_n\to x$ in $X$, the pointwise distance $d(x,[c_1,c_2])$ is at most $|x-x_n|$, so its $L^2$ norm is zero in the limit and $x\in C$. Hence $f\in\Gamma_0(X)$. For every $y\in C$, the pointwise inequality $|\pi(x)-x|\le|y-x|$, squared and integrated, proves minimality of $\pi\circ x$ in the proximal objective. Strict convexity of the squared distance on $C$ gives uniqueness. Therefore, for every $h>0$,
\[
 \operatorname{prox}_{hf}(x)=\pi\circ x.
\]
For a fixed orthonormal basis and $R\ge1$, the activation on $E_R$ is $\sigma_{f,h,R}(x)=P_R(\pi\circ x)$. Its feasibility after projection requires a separate preservation condition.

The nonempty-set assumption holds if $0\in[c_1,c_2]$, or if $\mu(\Omega)<\infty$. If $\mu(\Omega)=\infty$ and $0\notin[c_1,c_2]$, every function with values in this interval has infinite squared integral, so $C$ is empty.
\end{example}

For every example, a matrix acts on encoded coordinates by
\[
 \bigl(Au^{\downarrow:R}\bigr)^{\uparrow:R}
 =\sum_{i=0}^{R-1}\sum_{j=0}^{R-1}A_{ij}\langle u,e_j\rangle_Xe_i.
\]
The row index labels the output coordinate. This convention agrees with Definition~\ref{def:PEO}.
\subsection{A Finite-Dimensional Application: Learning a minimization operator}
\label{s:Applications__ss:Convex}

As an additional sanity check, we first consider a splitting problem over a finite dimensional Hilbert space. For the Hilbert space $X = \mathbb{R}^d$ and a nonempty closed convex subset $C \subseteq \mathbb{R}^d$, we aim to learn the operator
\begin{equation*}
    \left\lbrace g: \mathbb{R}^d \rightarrow \mathbb{R} \text{ is convex} \right\rbrace \ni g \quad \mapsto \quad \argmin_{x \in C} g(x) = \argmin_{x \in \mathbb{R}^d} \left( f(x) + g(x) \right) \subseteq C,
\end{equation*}
where $f: \mathbb{R}^d \rightarrow (-\infty,\infty]$ is defined by $f(x) = 0$ if $x \in C$, and $f(x) = \infty$ otherwise. In this case, the proximal operator of $f$ is given as $\operatorname{prox}_f(x) = \operatorname{proj}_C(x) \eqdef \argmin_{y \in C} \Vert x-y \Vert
$, for all $x\in \mathbb{R}^d$.

For $C=[-1,1]^d$ and continuous $g$, the optimizer set is nonempty: a minimizing sequence has a convergent subsequence in the compact cube, and continuity makes its limit a minimizer. Uniqueness is not asserted.
\begin{figure}[ht]
    \begin{minipage}[c]{0.49\textwidth}
        \centering
        \includegraphics[width=1.0\linewidth]{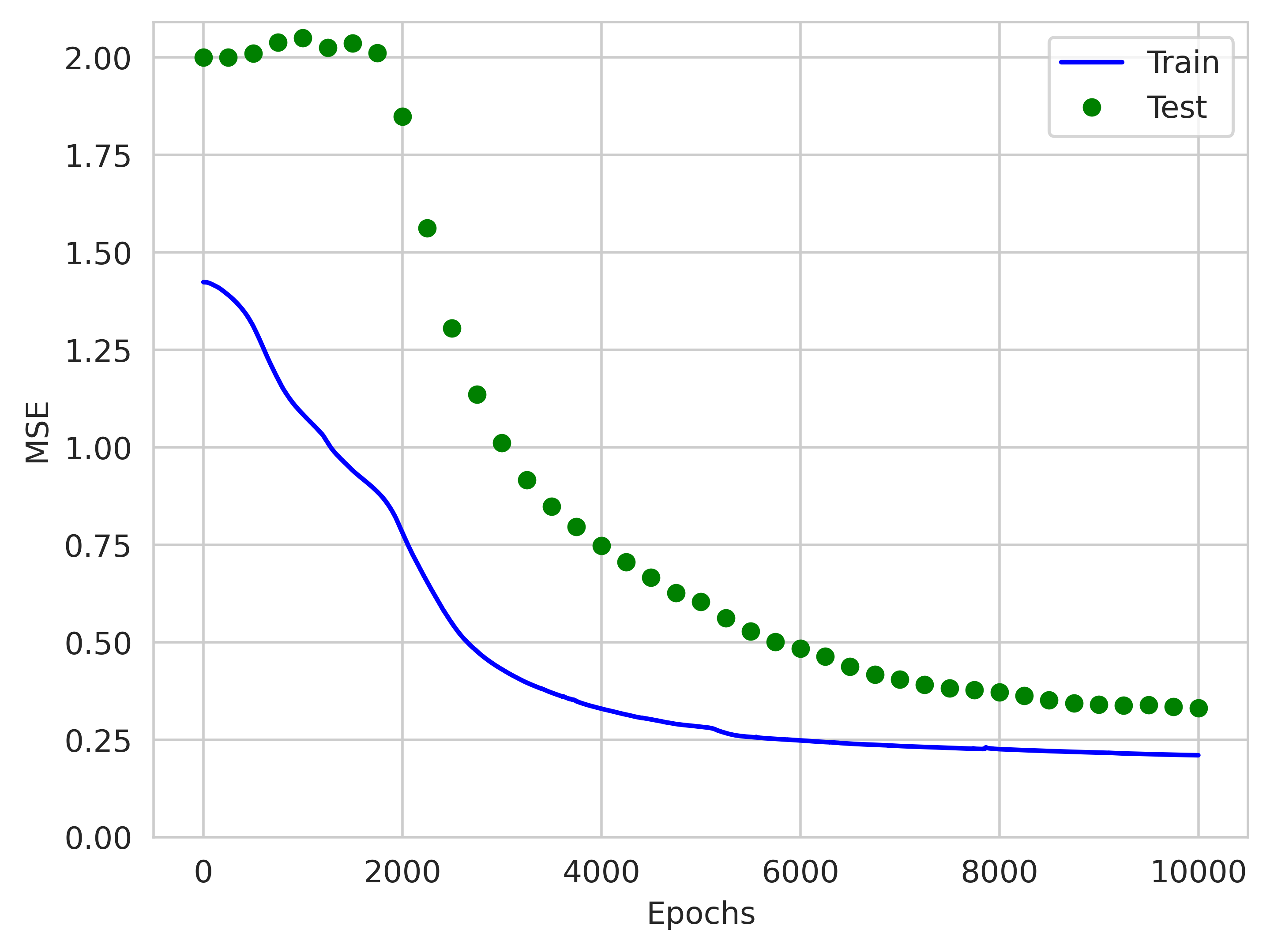}

        \subcaption{Learning performance}
    \end{minipage}
    \begin{minipage}[c]{0.49\textwidth}        
        \centering
        \vspace{0.1cm}
        
        \begin{tabular}{r|r|r|}
             \multicolumn{1}{l|}{$k$} & \multicolumn{1}{l|}{True} & \multicolumn{1}{l|}{Predict} \\
             \hline
             \input{python/min_op_test.txt}
        \end{tabular}
        
        \vspace{0.65cm}

        \subcaption{Solution of \eqref{EqMinOp} for four functions $g_k$ of the test set.}
    \end{minipage}
    \caption{Learning the minimization operator $\mathcal{S}$ in \eqref{EqMinOp} by a Proximal Equilibrium Operator $\mathcal{G}$. In (a), the learning performance is displayed in terms of the mean squared error (MSE) $\frac{1}{\vert K \vert} \sum_{k \in K} \Vert \mathcal{S}(g_k) - \mathcal{G}(g_k) \Vert^2$ on the training set (label ``Train'') and test set (label ``Test''). In (b), the predicted solution $\mathcal{G}(g_k)$ (label ``Predict'') is compared to the true solution $\mathcal{S}(g_k)$ (label ``True'') for four $k$ of the test set.}
    \label{FigMinOp}
\end{figure}

\begin{example}
    For $d = 2$, we consider the Hilbert space $X = \mathbb{R}^d$ and the convex subset $C = [-1,1]^d$. In this setting, we aim to learn the minimization operator
    
\begin{equation}
\label{EqMinOp}
 g\mapsto\mathcal M(g)=\operatorname*{argmin}_{x\in[-1,1]^d}g(x)
 =\operatorname*{argmin}_{x\in\mathbb R^d}\bigl(\iota_{[-1,1]^d}(x)+g(x)\bigr)
\end{equation}

    by a Proximal Equilibrium Operator $\mathcal{G}$ of rank $R = d = 2$, depth $L = 20$, and sample points $M = 20$. Hereby, the function $f: \mathbb{R}^d \rightarrow (-\infty,\infty]$ is defined as above. 
    
    To this end, we choose the standard orthonormal basis of $\mathbb{R}^d$. Moreover, we apply the Adam algorithm over 10000 epochs with learning rate $2 \cdot 10^{-4}$ to train the Proximal Equilibrium Operator on a training set consisting of 9000 convex functions $\mathbb{R}^d \ni x \mapsto g_k(x) \eqdef \frac{1}{2} x^\top A_k x + b_k^\top x + c_k \in (-\infty,\infty)$, $k = 1,...,9000$, 
where $A_k \in \mathbb{S}^d_+$, $b_k \in \mathbb{R}^d$, and $c_k \in \mathbb{R}$ are randomly initialized. For $A_k\ne0$, this quadratic loss is outside the global gradient-tail class: choose $v$ with $A_kv\ne0$, and then $\nabla g_k(tv)=tA_kv+b_k$ is unbounded as $t\to\infty$.
 
In addition, we evaluate its generalization performance every 250-th on a test set consisting of 1000 functions $D_k \ni x \mapsto g_k(x) \eqdef \ln\big( \sum_{i=1}^d \exp(b_{k,i} x_i) + c_k \big) \in (-\infty,\infty)$, $k = 9001,...,10000$, where $b_k \eqdef (b_{k,1},...,b_{k,d})^\top \in \mathbb{R}^d$ (with either $b_{k,i} \geq 0$ for all $i=1,...,d$, or $b_{k,i} \leq 0$ for all $i=1,...,d$) and $c_k \in \mathbb{R}$ are also randomly initialized. Here $D_k\eqdef\{x\in\mathbb R^d: c_k+\sum_{i=1}^d e^{b_{k,i}x_i}>0\}$. The results are reported in Figure~\ref{FigMinOp}.

A sufficient condition for $D_k=\mathbb R^d$ and convexity is $c_k\ge0$. Put $Z_k(x)=c_k+\sum_{i=1}^d e^{b_{k,i}x_i}$.
To prove that sufficient condition, let $a_i=b_{k,i}e_i$ for $1\le i\le d$, where here $e_1,\ldots,e_d$ denote the usual coordinate vectors, and let $a_0=0$. Define
\[
 p_0(x)=c_k/Z_k(x),\qquad p_i(x)=e^{\langle a_i,x\rangle}/Z_k(x),
 \qquad \bar a(x)=\sum_{i=0}^d p_i(x)a_i.
\]
For $c_k\ge0$, $Z_k>0$ everywhere, $p_i\ge0$, and $\sum_i p_i=1$. Differentiation gives
\[
 \nabla^2 g_k(x)=\sum_{i=0}^d p_i(x)a_ia_i^\top-\bar a(x)\bar a(x)^\top.
\]
For every $v$, its quadratic form is
$\sum_i p_i(\langle a_i,v\rangle-\langle\bar a,v\rangle)^2\ge0$.
Integrating the second directional derivative along a segment proves convexity. Arbitrary negative $c_k$ do not have this property: for $d=1$, $b=1$, $c=-1$, the domain is $(0,\infty)$ and $g''(x)=-e^x/(e^x-1)^2<0$ there.

\end{example}

\section*{Acknowledgements and funding}
\label{s:Acknow}
A.~Kratsios gratefully acknowledges financial support from the Natural Sciences and Engineering Research Council of Canada (NSERC) through Discovery Grant Nos.\ RGPIN-2023-04482 and DGECR-2023-00230.
A.~Neufeld gratefully acknowledges financial support by the MOE AcRF Tier 2 Grant MOE-T2EP20225-0011. P.~Schmocker gratefully acknowledges financial support by the FinsureTech Hub of ETH Zurich. 
We further acknowledge that resources used in the preparation of this research were provided, in part, by the Province of Ontario, the Government of Canada through CIFAR, and the industry sponsors of the \href{https://vectorinstitute.ai/partnerships/current-partners/}{Vector Institute}. This work was made possible by the facilities of the Shared Hierarchical Academic Research Computing Network (\href{https://www.sharcnet.ca}{SHARCNET}) and the Digital Research Alliance of Canada (\href{https://alliancecan.ca/en}{https://alliancecan.ca/en}).

\bibliographystyle{siamplain}
\bibliography{Bookkeaping/Refs}

\end{document}